\relax

\newif\ifarxiv\arxivtrue

\newif\ifnotarxiv
\ifarxiv \notarxivfalse \else \notarxivtrue \fi

\ifarxiv
\newcommand{\arxiv}[1]{#1}
\else
\newcommand{\arxiv}[1]{}
\fi

\ifarxiv
\newcommand{\notarxiv}[1]{}
\else
\newcommand{\notarxiv}[1]{#1}
\fi

\ifarxiv
    \documentclass[11pt]{article}
\else
    \documentclass{article}
\fi
\usepackage[margin=1in]{geometry}

\ifarxiv
    \usepackage[authoryear]{natbib}
\else
    \usepackage{natbib}  
\fi

\usepackage[table]{xcolor}
\ifarxiv
    \usepackage{color}
    \definecolor{cornellred}{rgb}{0.7, 0.11, 0.11}
    \definecolor{dgreen}{rgb}{0.0, 0.5, 0.0}
    \definecolor{ballblue}{rgb}{0.13, 0.67, 0.8}
    \definecolor{royalblue(web)}{rgb}{0.25, 0.41, 0.88}
    \definecolor{bleudefrance}{rgb}{0.19, 0.55, 0.91}
    \definecolor{royalazure}{rgb}{0.0, 0.22, 0.66}
    \usepackage[breaklinks]{hyperref}
    \hypersetup{
        colorlinks = true,
        linkcolor = cornellred,
        citecolor = royalazure,
        linkbordercolor = white,
        urlcolor  = cornellred
    }
\fi

\newcommand*\samethanks[1][\value{footnote}]{\footnotemark[#1]}
\ifarxiv
\author{
    Dmitrii Avdiukhin\thanks{Indiana University, Bloomington. Research supported by NSF award CCF-1657477 and Facebook Faculty Research Award.}\\{\tt davdyukh@iu.edu}
    \and
    Grigory Yaroslavtsev\samethanks[1]\\{\tt grigory@grigory.us}
}
\fi

\usepackage{microtype}
\usepackage{graphicx}
\usepackage{subcaption}
\usepackage{booktabs} 
\usepackage[ruled,lined,linesnumbered]{algorithm2e}

\usepackage{hyperref}

\title{Escaping Saddle Points with Compressed SGD}

\usepackage[utf8]{inputenc} 
\usepackage[T1]{fontenc}    
\usepackage{hyperref}       
\usepackage{url}            
\usepackage{booktabs}       
\usepackage{amsfonts}       
\usepackage{nicefrac}       
\usepackage{amsmath,amssymb,amsfonts}
\usepackage{nicefrac}

\usepackage{tabu}
\usepackage{makecell}
\usepackage{paralist}
\usepackage{verbatim}
\usepackage{multirow}
\usepackage[framemethod=tikz]{mdframed}

\let\oldnl\nl
\newcommand{\nonl}{\renewcommand{\nl}{\let\nl\oldnl}}

\newcommand{\multiline}[1]{
        \begin{tabular}{@{}c@{}}
        #1
        \end{tabular}
}

\newcommand{\defthm}[1]{\newmdtheoremenv[roundcorner=10, innerleftmargin=7, innerrightmargin=7, leftmargin=-7, rightmargin=-7, backgroundcolor=#1!1.5, nobreak=true]}

\defthm{blue}{theorem}{Theorem}[section]
\newtheorem{corollary}[theorem]{Corollary}
\defthm{yellow}{lemma}[theorem]{Lemma}
\newtheorem{proposition}[theorem]{Proposition}
\defthm{green}{definition}[theorem]{Definition}
\defthm{red}{assumption}{Assumption}

\newenvironment{proof}{\noindent{\bf Proof : \ }}{\hfill$\Box$\par\medskip}

\newcommand{\eps}{\varepsilon}
\newcommand{\fmax}{f_{\max}}
\newcommand{\Exp}[1]{\mathbb{E}\left[#1\right]}
\newcommand{\ExpArg}[2]{\mathbb{E}_{#1}\left[#2\right]}

\newcommand{\vx}{\mathbf{x}}
\newcommand{\vy}{\mathbf{y}}
\newcommand{\vv}{\mathbf{v}}
\newcommand{\vg}{\mathbf{v}}
\newcommand{\tvx}{\vy}
\newcommand{\g}{\mathbf{g}}
\newcommand{\err}{\mathbf{e}}
\newcommand{\step}{\eta}
\newcommand{\ourNoise}{\xi}
\newcommand{\sgdNoise}{\zeta}
\newcommand{\fullNoise}{\psi}
\newcommand{\gradLip}{L}
\newcommand{\hesLip}{\rho}
\newcommand{\sgdLip}{\tilde \ell}

\newcommand{\std}{\sigma}
\newcommand{\fullStd}{\chi}
\newcommand{\compr}{\mu}
\newcommand{\escIter}{\mathcal I}
\newcommand{\df}{\mathcal F}
\newcommand{\escRad}{\mathcal R}
\newcommand{\eigen}{\gamma}
\newcommand{\lmin}{\lambda_{\min}}

\newcommand{\C}{\mathcal C}
\newcommand{\R}{\mathbb R}
\newcommand{\workers}{\mathcal W}
\newcommand{\sumvar}{\beta}

\newcommand{\sgdRand}{\theta}
\newcommand{\comprRand}{\tilde \theta}
\newcommand{\sgdDistr}{\mathcal D}
\newcommand{\comprDistr}{\tilde{\mathcal D}}

\newcommand{\randomk}{\textsc{RandomK}\xspace}
\newcommand{\topk}{\textsc{TopK}\xspace}

\newcommand{\cf}{c_\df}
\newcommand{\cstep}{c_\step}
\newcommand{\cit}{c_\escIter}
\newcommand{\crad}{c_\escRad}
\newcommand{\cnoise}{c_r}

\newcommand{\sqrtre}{\sqrt{\hesLip \eps}}
\newcommand{\sqrtomc}{\sqrt{1 - \compr}}
\newcommand{\reseterr}{reset\_error}

\begin{document}

\maketitle

\begin{abstract}%
Stochastic gradient descent (SGD) is a prevalent optimization technique for large-scale distributed machine learning.
While SGD computation can be efficiently divided between multiple machines, communication typically becomes a bottleneck in the distributed setting. 
Gradient compression methods can be used to alleviate this problem, and a recent line of work shows that SGD augmented with gradient compression converges to an $\eps$-first-order stationary point. 
In this paper we extend these results to convergence to an $\eps$-\textit{second}-order stationary point ($\eps$-SOSP), which is to the best of our knowledge the first result of this type.
In addition, we show that, when the stochastic gradient is not Lipschitz, compressed SGD with \randomk compressor converges to an $\eps$-SOSP with the same number of iterations as uncompressed SGD~\citep{jin2019nonconvex} (JACM), while improving the total communication by a factor of $\tilde \Theta(\sqrt{d} \eps^{-\nicefrac 34})$, where $d$ is the dimension of the optimization problem.
We present additional results for the cases when the compressor is arbitrary and when the stochastic gradient is Lipschitz.
\end{abstract}



\section{Introduction}
\label{sec:intro}



Stochastic Gradient Descent (SGD) and its variants are the main workhorses of modern machine learning.
Distributed implementations of SGD on a cluster of machines with a central server and a large number of workers are frequently used in practice due to the massive size of the data.
In distributed SGD each machine holds a copy of the model and the computation proceeds in rounds. In every round, each worker finds a stochastic gradient based on its batch of examples, the server averages these stochastic gradients to obtain the gradient of the entire batch, makes an SGD step, and broadcasts the updated model parameters to the workers.
With a large number of workers, computation parallelizes efficiently while communication becomes the main bottleneck~\citep{chilimbi2014project,strom2015scalable}, since each worker needs to send its gradients to the server and receive the updated model parameters.
Common solutions for this problem include: local SGD and its variants, when each machine performs multiple local steps before communication~\citep{stich2018local}; decentralized architectures which allow pairwise communication between the workers~\citep{mcmahan2017communication} and gradient compression, when a compressed version of the gradient is communicated instead of the full gradient~\citep{bernstein2018signsgd,stich2018sparsified,karimireddy2019error}.
In this work, we consider the latter approach, which we refer to as \emph{compressed SGD}.

Most machine learning models can be described by a $d$-dimensional vector of parameters $\vx$ and the model quality can be estimated as a function $f(\vx)$.
Hence optimization of the model parameters can be cast a minimization problem $\min_\vx f(\vx)$, where $f\colon \R^d \to \R$ is a continuous function, which can be optimized using continuous optimization techniques, such as SGD.
Fast convergence of compressed SGD to a first-order stationary point (FOSP, $\|\nabla f(\vx)\| < \eps$) was shown recently for various gradient compression schemes~\citep{bernstein2018signsgd,stich2018sparsified,karimireddy2019error,ivkin2019communication,alistarh2017qsgd}.
However, even an exact FOSP can be either a local minimum, a saddle point or a local maximum.
While local minima often correspond to good solutions in machine learning applications~\citep{geLM16matrix,sunQW16phase,bhojanapalliNS16matrix}, saddle points and local maxima are always suboptimal and it is important for an optimization algorithm to avoid converging to them.
In particular, \citet{choromanska2015loss} show that for neural networks many local minima are almost optimal, but the corresponding loss functions have a combinatorial explosion in the number of saddle points.
Furthermore, \citet{dauphin2014identifying} show that saddle points can significantly slow down SGD convergence and hence it is important to be able to escape from them efficiently.

Since finding a local minimum is NP-hard in general~\citep{anandkumarG16local}, a common relaxation of this requirement is to find an approximate second-order stationary point (SOSP), i.e. a point with a small gradient norm ($\|\nabla f(\vx)\| < \eps$) and the smallest (negative) eigenvalue being small in absolute value ($\lmin(\nabla^2 f(\vx)) > -\eps_H$).
When $f$ has $\hesLip$-Lipschitz Hessian (i.e. $\|\nabla^2 f(x) - \nabla^2 f(y)\| \le \hesLip \|x - y\|$ for all $x,y$), a standard choice of $\eps_H$ is $\sqrt{\hesLip \eps}$~\citep{nesterovP06}, and such approximate SOSP is commonly referred as an $\eps$-SOSP.
While second-order optimization methods allow one to escape saddle points, such methods are typically substantially more expensive computationally.
A line of work originating with the breakthrough of~\citet{geHJY15} shows that first-order methods can escape saddle points when perturbations are added at certain iterations.
In particular, a follow-up~\citet{jin2019nonconvex} show that SGD converges to an $\eps$-SOSP in an almost optimal number of iterations.

In this paper, we show that even \emph{compressed} SGD can efficiently converge to an $\eps$-SOSP.
To the best of our knowledge, this is the first result showing convergence of compressed methods to a second-order stationary point.

\subsection{Related Work}
\label{sec:related_work}


\paragraph{Escaping from saddle points}

While it is known that gradient descent with random initialization converges to a local minimum almost surely~\citep{lee2016gradient}, existence of saddle points may result in exponential number of steps with non-negligible probability~\citep{du2017gradient}.
Classical approaches for escaping from saddle points assume access to second-order information~\citep{nesterovP06, curtis2014trust}.
Although these algorithms find a second-order stationary point (SOSP) in $O(\eps^{-\nicefrac 32})$ iterations,
each iteration requires computation of the full Hessian matrix, which can be prohibitive for high-dimensional problems in practice.
Some approaches relax this requirement, and instead of full Hessian matrix they only require access to a Hessian-vector product oracle~\citep{CarmonD16gradient, agarwal2017finding}.
While in certain settings, including training of neural networks, it's possible to compute Hessian-vector products (HVP) efficiently~\citep{Pearlmutter94product,schraudolph2002product}, such an oracle might not be available in general.
Furthermore, in practice HVP-based approaches are significantly more complex compared to SGD (especially if the workers aren't communicating in every iteration in the distributed setting) and require additional hyperparameter tuning.
Moreover, HVP is typically used for approximate an eigenvector computation, which in practice may increase the number of iterations by a logarithmic factor.

Limitations of second-order methods motivate a long line of recent research on escaping from saddle points using first-order algorithms, starting from~\citet{geHJY15}.
\citet{JinGNKJ17} show that perturbed gradient descent finds $\eps$-SOSP in $\tilde{O}(\eps^{-2})$ iterations.
Later, this is improved by a series of accelerated algorithms~\citep{carmon2016accelerated,agarwal2017finding,carmon2017convex,jin2018accelerated} which achieves $\tilde O\left(\eps^{- \nicefrac 74}\right)$ iteration complexity.
There are also a number of algorithms designed for finite sum setting where $f(x) = \sum_{i=1}^n f_i(x)$~\citep{reddi2017generic,allen2018neon2,fang2018spider}, or in case when only stochastic gradients are available~\citep{tripuraneni2018stochastic,jin2019nonconvex}, including variance reduction techniques~\citep{allen2018natasha,fang2018spider}.
The sharpest rates in these settings have been obtained by~\citet{fang2018spider},~\citet{zhou2019stochastic} and~\citet{Fang2019sharp}.

\paragraph{Compressed SGD}

While gradient compression may require a complex communication protocol, from theoretical perspective this process is often treated as a black-box function: a (possibly randomized) function $\C$ is called a $\compr$-compressor if $\Exp{\|\vx - \C(\vx)\|^2} < (1 - \compr) \|\vx\|^2$.
In a simplified form, the update step in compressed SGD can be expressed as $\vx_{t+1} \gets \vx_t - \step \C(\nabla f(\vx))$\footnote{We emphasize that the actual update equation is more complicated, see Algorithm~\ref{alg:arbitrary_compressor}.}.
Notable examples of compressors include the following:

\textbf{\textsc{Sign}} function $\C(\vx) = \frac {\|x\|_1} d sign(x)$ is a $\nicefrac 1d$-compressor~\citep{bernstein2018signsgd}. Representation of $\C(x)$ requires $O(d)$ bits, but it is hard to compute in distributed settings: it's not clear how to find the signs of the coordinates without knowing the full vector, which requires each worker to send all coordinates.
A practical solution is for each worker to communicate \textsc{Sign} of its local gradient, and the final sign for each coordinate is selected by majority vote.
Unfortunately, the resulting vector is not necessarily a compression of the gradient.

\textbf{\textsc{Quantization}}~\citep{alistarh2017qsgd} uniformly splits segment $[0, \|\vx\|]$ into $s$ buckets of the same size. Let $\ell_i = \left\lfloor \frac {|\vx_i|} {\|\vx\| / s} \right\rfloor$; then $|\vx_i|$ is randomly rounded to one of $\ell_i \frac {\|\vx\|} {s}$ and $(\ell_i + 1) \frac {\|\vx\|} {s}$. The compressor returns non-zero coordinates after rounding. For $s = 1$, \textsc{Quantization} $Q(\vx)$ can be represented using $\tilde O(\sqrt d)$ bits.
While it doesn't fall into the compression framework, since $\|Q(\vx) - \vx\|$ can be much greater than $\|\vx\|$, it has a property $\Exp{Q(\vx)} = \vx$, which allows one to show convergence.

\textbf{\textsc{TopK}} function preserves only $k$ largest (by the absolute value) coordinates of a vector and is a $\nicefrac kd$-compressor~\citep{stich2018sparsified}. This compressor can be represented using $\tilde O(k)$ bits, but similarly to \textsc{Sign}, it is hard to compute in distributed settings.
To address this issue, \citet{alistarh2018convergence} assume that \textsc{TopK} of the average gradient is close to the average of \textsc{TopK} of local gradients and show that this assumption holds in practice.

\textbf{Sketch-based \textsc{TopK}}~\citep{ivkin2019communication} is randomized communication-efficient compressor based on Count Sketch, which recovers top-$k$ coordinates in a distributed setting.
It uses the fact that Count Sketch is a linear sketch (and therefore it can be easily combined across multiple machines) and can be used to recover top-$k$ coordinates of the vector with high probability.
Therefore, it can be used as an efficient $\nicefrac kd$-compressor requiring $\tilde O(k)$ communication.

\textbf{\randomk} compressor preserves $k$ random coordinates of a vector.
It is a $\nicefrac kd$-compressor~\citep{stich2018sparsified} requiring $O(k)$ communication.

While it was shown that SGD with compressor converges (e.g.~\citet{karimireddy2019error} and the works above), the convergence was shown only to a FOSP.
The crucial idea to facilitate convergence is to use error-feedback~\citep{stich2018sparsified}: the difference between the actual gradient and the compressed gradient is propagated to the next iteration.

\subsection{Our Contributions}

Our main contribution is the analysis showing that perturbed compressed SGD with error-feedback can escape from saddle points efficiently.
Moreover, we show faster convergence rate for a certain type of compressors and show that such compressors exist.
Inspired by the ideas from~\citet{jin2019nonconvex} and~\citet{stich2018sparsified}, we present an algorithm (Algorithm~\ref{alg:arbitrary_compressor}) which uses perturbed compressed gradients with error-feedback and converges to an $\eps$-second-order stationary point (see Theorem~\ref{thm:sosp}).
Our main results shows that compressed SGD with \randomk compressor achieves substantial communication improvement:

\begin{theorem}[Informal, Theorem~\ref{thm:sosp_randomk} and Corollary~\ref{cor:randomk_comm}]
    \label{thm:intro_randomk}
    Assume that $f$ has Lipschitz gradient and Lipschitz Hessian.
    Let $\alpha = 1$ when the stochastic gradient is Lipschitz and $\alpha=d$ otherwise.
    Then SGD with \randomk compressor (which selects $k$ random coordinates) with $k = \frac {d \eps^{\nicefrac 34}} {\sqrt \alpha}$ converges to an $\eps$-SOSP after $\tilde O\left(\frac \alpha {\eps^4}\right)$ iterations, with $\tilde O\left(\frac {d \sqrt \alpha} {\eps^{3 + \nicefrac 14}}\right)$ total communication per worker.
\end{theorem}
Compared with the uncompressed case, the  total communication improves by $\eps^{-\nicefrac 14}$ when the stochastic gradient is Lipschitz and by $\sqrt d \eps^{-\nicefrac 34}$ otherwise (the sharpest results for SGD are by~\cite{Fang2019sharp} and~\cite{jin2019nonconvex} respectively).
In Theorem~\ref{thm:intro_randomk}, we heavily rely on the following property of \randomk: when its randomness (i.e. sampled $k$ coordinates) is fixed, the compressor becomes a linear function.
For other compressors, this property doesn't necessarily hold; in this case, we show convergence with a slower convergence rate:

\begin{theorem}[Informal, Theorem~\ref{thm:sosp} and Corollary~\ref{cor:table_comms}]
    Assume that $f$ has Lipschitz gradient and Lipschitz Hessian.
    Let $\alpha = 1$ when the stochastic gradient is Lipschitz and $\alpha=d$ otherwise.
	Let $\C$ be a $\nicefrac kd$-compressor requiring $\tilde O(k)$ communication.
    Then SGD with compressor $\C$ with $k = \frac {d \sqrt d \eps^{\nicefrac 34}} {\sqrt \alpha}$ converges to an $\eps$-SOSP after $\tilde O\left(\frac \alpha {\eps^4}\right)$ iterations, with $\tilde O\left(\frac {d \sqrt d \sqrt \alpha} {\eps^{3 + \nicefrac 14}}\right)$ total communication per worker.
\end{theorem}
Compared with the uncompressed case, the total communication improves by $\frac {\eps^{-\nicefrac 14}} {\sqrt d}$ when the stochastic gradient is Lipschitz (note that this is the only setting where the convergence improvement is conditional, requiring $\eps = o(d^{-2})$) and by $\eps^{-\nicefrac 34}$ otherwise.
Table~\ref{tab:results} in Section~\ref{sec:convergence} outlines communication improvements for various choices of compression parameters.
We outline our main techniques and technical contributions in Section~\ref{sec:proof_sketch} and present the complete proof in Appendix~\ref{sec:fosp} and~\ref{sec:sosp}.

\section{Preliminaries}
\label{sec:prelims}

\paragraph{Function Properties} For a twice differentiable nonconvex function $f\colon \R^d \to \R$, we consider the unconstrained minimization problem
$\min_{\vx \in \R^d} f(\vx)$. 

We use the following standard~\citep{jin2019nonconvex, Fang2019sharp, xu2018first, allen2018natasha, zhou2018finding} assumptions about the objective function $f$:

\begin{assumption}
\label{ass:lip} $f$ is $\fmax$-bounded, $\gradLip$-smooth and has $\hesLip$-Lipschitz Hessian, i.e. for all $x, y$:
\[
	|f(\vx) - f(\vy)| \le \fmax,\quad
	\|\nabla f(\vx) - \nabla f(\vy)\| \le \gradLip \|\vx-\vy\|,\quad
    \|\nabla^2 f(\vx) - \nabla^2 f(\vy)\| \le \hesLip \|\vx-\vy\|
\]
\end{assumption}


%

\begin{assumption}
\label{ass:sg} Access to an unbiased stochastic gradient oracle $\nabla F(\vx, \sgdRand)$, whose randomness is controlled by a parameter $\sgdRand \sim \sgdDistr$\footnote{E.g. $\sgdRand$ is a minibatch selected at the current iteration}, with bounded variance:
\begin{align*}
    \ExpArg{\sgdRand \sim \sgdDistr}{\nabla F(\vx, \sgdRand)} = \nabla f(\vx),
    \quad \ExpArg{\sgdRand \sim \sgdDistr}{\|\nabla F(\vx, \sgdRand) - \nabla f(\vx)\|^2} \le \std^2
\end{align*}
\end{assumption}
As shown by the above works, smoothness allows one to achieve fast convergence for nonconvex optimization problems (namely, to use the folklore descent lemma).
Similarly, Lipschitz Hessian allows one to show fast second-order convergence, since, within a certain radius, the function stays close to its quadratic approximation (see e.g. \citet[Section 9.5.3]{boyd2004convex}).
As common in the literature, in our convergence rates we treat $\gradLip$, $\hesLip$ and $\std^2$ as constants.

We consider an additional optional assumption~\citep{jin2019nonconvex}:
\begin{assumption}[Optional]
\label{ass:lipsg} Lipschitz stochastic gradient. For any $\vx,\vy,\sgdRand$:
\begin{align*}
\|\nabla F(\vx, \sgdRand) - \nabla F(\vy, \sgdRand)\| \le \sgdLip \|\vx - \vy\|
\end{align*}
\end{assumption}

From machine learning perspective, Assumption~\ref{ass:lipsg} means that for the same mini-batch, if the initial models are close, their updates are also close.
For neural networks, since each network layer is a composition of an activation function and a linear function, the assumption holds when activation functions are Lipschitz (note however that $\sgdLip$ may grow exponentially with the number of layers).

\paragraph{Gradient Compression} Our goal is to optimize $f$ in a distributed setting~\citep{dekel2012optimal,li2014communication}: given $\workers$ workers, for each worker $i$ we have a corresponding data distribution $\sgdDistr_i$.
Then the each worker has a corresponding function $f_i(\vx) = \ExpArg{\sgdRand \sim \sgdDistr_i}{F(\vx, \sgdRand)}$ and $f=\sum_{i=1}^\workers f_i$.
In a typical distributed SGD setting, each worker computes a stochastic gradient $\nabla F_i(\vx, \sgdRand_i)$ and sends it to the coordinator machine.
The coordinator machine computes the average of these gradients $\vg = \frac 1 {\workers} \sum_{i=1}^\workers \nabla F_i(\vx, \sgdRand_i)$ and broadcasts it to the workers, which update the local parameters $\vx \gets \vx - \step \vg$ ($\step$ is the step size).

With this approach, with increase of the number of machines, the computation can be perfectly parallelized.
However, with each machine required to send its gradient, communication becomes the main bottleneck~\citep{chilimbi2014project,strom2015scalable}.
There exist various solutions to this problem (see Section~\ref{sec:intro}), including gradient compression, when each machine sends an approximation of its gradient.
Then coordinator averages these approximations and broadcasts the average to all machines (possibly compressing it again, see discussion on \topk and \textsc{Sign} in Section~\ref{sec:related_work}).

Depending on the compression method, this protocol provides different gradient approximation and different communication per machine.
There is a natural trade-off between approximation and communication, and it's not clear whether having smaller per-iteration communication results in smaller total communication required for convergence.
The approximation quality can be formalized using the following definition:
\begin{definition}[\citet{stich2018sparsified}]
	Function $\C(\vx, \comprRand)$, whose randomness is controlled by a parameter $\comprRand \sim \comprDistr$\footnote{E.g. for \randomk, $\comprRand$ is the set of indices of coordinates. When the compressor is deterministic, we omit $\comprRand$.}, is a \emph{$\compr$-compressor} if
	\[\ExpArg{\comprRand \sim \comprDistr}{\|\vx - \C(\vx, \comprRand)\|^2} < (1 - \compr) \|\vx\|^2\]
\end{definition}

Section~\ref{sec:related_work} provides examples of important compressors.
In our analysis, we consider two cases: a general compressor and a linear compressor, and in the latter case, we show an improved convergence rate.
\begin{definition}
    \label{def:linear_compr}
    $\C$ is a \emph{linear compressor} if $\C(\cdot, \comprRand)$ is a linear function for any $\comprRand$.
\end{definition}
One example of a linear compressor is \randomk, which preserves $k$ random coordinates of a vector; it's a $\nicefrac kd$-compressor~\citep{stich2018sparsified} and it's trivial to compute in the distributed setting (assuming shared randomness).

\paragraph{Stationary Points}
The optimization problem of finding a global minimum or even a local minimum is NP-hard for nonconvex objectives~\citep{nesterov00,anandkumarG16local}.
Instead, as is standard in the literature, we show convergence to an approximate first-order stationary point or an approximate second-order stationary point, see Section~\ref{sec:intro}.

\begin{definition}
	For a differentiable function $f$, $\vx$ is an $\eps$-first-order stationary point ($\eps$-FOSP) if $\|\nabla f(\vx)\| \le \eps$.
\end{definition}
An $\eps$-FOSP can be a local maximum, a local minimum or a saddle point.
While local minima typically correspond to good solutions, saddle points and local maxima are inherently suboptimal. Assuming non-degeneracy, saddle points and local maxima have escaping directions, corresponding to Hessian's negative eigenvectors.
Following~\citet{nesterovP06} we refer to points with no escape directions (up to a second-order approximation) as approximate second-order stationary points:

\begin{definition}[\citet{nesterovP06}]
	For a twice-differentiable, $\hesLip$-Hessian Lipschitz function $f$,
	$\vx$ is an $\eps$-second-order stationary point ($\eps$-SOSP) if $\|\nabla f(\vx)\| \le \eps$ and $\lmin(\nabla^2 f(\vx)) \ge -\sqrt{\rho \eps}$\footnote{While one can consider two threshold parameters~-- $\eps_g$ for $\nabla f$ and $\eps_H$ for $\nabla^2 f$~-- we follow convention of~\citet{nesterovP06} which selects $\eps_H = -\sqrt{\rho \eps}$, which, intuitively, balances first-order and second-order variability.}, where $\lmin$\footnote{In this work, $\compr$ always denotes compression parameter and $\lmin$ denotes the smallest eigenvalue.} is the smallest eigenvalue.
\end{definition}

An important property of points which are not $\eps$-SOSP is that they are unstable: adding a small perturbation allows gradient descent to escape them~\citep{geHJY15} (similar results were shown for e.g. stochastic~\citep{jin2019nonconvex} and accelerated~\citep{jin2018accelerated} gradient descent).
In this work we show that this property holds even for stochastic gradient descent with gradient compression.
\section{Algorithm and Analysis}
\label{sec:algorithm}

\paragraph{Algorithm} We present our algorithm in Algorithm~\ref{alg:arbitrary_compressor}, a compressed stochastic gradient descent approach based on~\citet[Algorithm 1]{stich2018sparsified}.
In order to achieve second-order convergence, similarly to~\citet{jin2019nonconvex}, we add artificial random noise $\ourNoise_t$ to gradient at every iteration, which allows compressed gradient descent to escape saddle points.

At every iteration $t$, we compute the stochastic gradient $\nabla F(\vx_t, \sgdRand_t)$.
Then we add artificial noise $\ourNoise_t$, compress the resulting value (Line~\ref{line:compression}) and update the current iterate $\vx_t$ using the compressed value (Line~\ref{line:update}).
However, the information is not lost during compression: the difference between the computed value and the compressed value (Line~\ref{line:error}), $\err_{t+1}$, is added to the gradient in the next iteration.
\citet{karimireddy2019error} show that carrying over the error term improves convergence of compressed SGD to a first-order stationary point.

Algorithm~\ref{alg:arbitrary_compressor} accepts an additional Boolean parameter $\reseterr$.
When this parameter is true, we set $\err_t$ to $0$ (Line~\ref{line:set_err_to_zero}) when conditions in Line~\ref{line:check_err_to_zero} hold:
either we moved far from the point where the condition was triggered last time (intuitively, the condition indicates that we successfully escaped from  a saddle point), or we spent a certain number of iterations since that event (to ensure that the accumulated compression error is sufficiently bounded).

\paragraph{Distributed Setting Considerations}
Algorithm~\ref{alg:arbitrary_compressor} provides a general framework for compressed SGD in distributed settings, with implementation details depending on the choice of the compressor function $\C$.
$\ourNoise_t$ can be efficiently shared between machines using shared randomness.
Each machine $i$ maintains its own local $\err_t^{(i)}$ which can be computed as $\err_{t+1}^{(i)} \gets \err_{t}^{(i)} + \nabla F_i(\vx_t, \sgdRand_t) + \ourNoise_t  - \g_t^{(i)}$. Then $\err_t = \frac 1\workers \sum_{i=1}^\workers \err_{t}^{(i)}$.
Finally, the norm in Line~\ref{alg:arbitrary_compressor} of Algorithm~\ref{line:set_err_to_zero} can be efficiently computed within multiplicative approximation using linear sketches.

\begin{algorithm}[thb]
	\SetKwInOut{Input}{input}
	\SetKwInOut{Output}{output}
	\SetKwRepeat{Do}{do}{while}
	\nonl \textbf{parameters:} $\step$~-- step size, $T$~-- number of iterations, $r^2$~-- variance of the artificial noise, {\color{blue} $\reseterr$~-- flag indicating whether compression error should be periodically reset to zero, $\escIter$~-- the number of iterations required for escaping, $\escRad$~-- escaping radius}  \\
	\Input{objective $f$, compressor function $\C$, starting point $\vx_0$}
	\Output{$\eps$-SOSP of $f$}
	\caption{Compressed SGD}
	\label{alg:arbitrary_compressor}
	$\err_0 \gets 0^d$\\
	{\color{blue}\lIf{$\reseterr$}{$t' \gets 0$}}
	\For {$t = 0 \ldots T-1$} {
	    {\color{blue}
	    // Reset the error after $\escIter$ iterations or in case we moved far from the initial point \\
	    \If {$\reseterr$ and ($t - t'> \escIter$ or $\|\vx_{t'} - (\vx_t - \step \err_t)\| > \escRad$)\label{line:check_err_to_zero}}{
	        $t' \gets t$,\ \ \  $\vx_{t} \gets \vx_t - \step \err_t$,\ \ \ $\err_t \gets 0^d$ \label{line:set_err_to_zero}
	    }}
	    Sample $\ourNoise_t \sim \mathcal N_d(0^d, r^2)$, \quad $\sgdRand_t \sim \sgdDistr$, \quad $\comprRand_t \sim \comprDistr$\\
		$\g_t \gets  \C(\err_t + \nabla F(\vx_t, \sgdRand_t) + \ourNoise_t, \comprRand_t)$\label{line:compression} \hfill // Compressed gradient\\
		$\vx_{t+1} \gets \vx_t - \step \g_t$ \label{line:update} \hfill // Compressed gradient descent step \\
		$\err_{t+1} \gets  \err_t + \nabla F(\vx_t, \sgdRand_t) + \ourNoise_t  - \g_t$ \label{line:error} \hfill // Error is the difference between compressed and uncompressed gradient\\
	}
	\Return{$\vx_T$}
\end{algorithm}

\subsection{Convergence to an \texorpdfstring{$\eps$}{eps}-FOSP}

In the following statements, $\tilde O$ hides polynomial dependence on $\gradLip, \hesLip, \fmax, \std, \sgdLip$ and polylogarithmic dependence on all parameters.
The first result is similar to that of~\citet{stich2018sparsified} (after reformulation in terms of $\eps$-FOSP), but is more general:
it covers the case when $\compr$ is close to $0$ and doesn't require any bounds on  $\|\nabla F(\vx, \sgdRand)\|$ or $\|\nabla f_i(\vx) - \nabla f(\vx)\|$, which are common assumptions in the literature (see Section~\ref{sec:related_work}).
The proof of the theorem is presented in Appendix~\ref{sec:fosp}.
\begin{theorem}[Convergence to $\eps$-FOSP]
	\label{thm:fosp}
	Let $f$ satisfy Assumptions~\ref{ass:lip} and~\ref{ass:sg} and let $\C$ be a $\compr$-compressor.
	Then for Algorithm~\ref{alg:arbitrary_compressor} with $\reseterr=false$ and $\step=\tilde O\left(\min\left(\eps^2, \frac \compr \sqrtomc \eps\right) \right)$,
	after $T=\tilde O(\frac 1 {\eps^2 \step}) = \tilde O\left(\frac 1 {\eps^4} + \frac {\sqrtomc} {\compr \eps^3}\right)$ iterations,
	at least half of visited points are $\eps$-FOSP\@.
\end{theorem}
\begin{corollary}
    For a $\nicefrac 1d$-compressor with $\tilde O(1)$ communication (polylogarithmic on all parameters), the total communication per worker is $\tilde O \left(\frac{1}{\eps^4} + \frac {d} {\eps^3}\right)$, which outperforms full SGD communication $\tilde O \left(\frac d {\eps^4}\right)$ by a factor of $\min \left(d, \eps^{-1}\right)$.
\end{corollary}



\subsection{Convergence to an \texorpdfstring{$\eps$}{eps}-SOSP}
\label{sec:convergence}

The next two theorems present our main result, namely that compressed SGD converges to an $\eps$-SOSP (see proof sketch in Section~\ref{sec:proof_sketch} and the full proof in Appendix~\ref{sec:sosp}).
The first theorem handles the case of a general compressor.
\begin{theorem}[Convergence to $\eps$-SOSP for general compressor]
	\label{thm:sosp}
	Let $f$ satisfy Assumptions~\ref{ass:lip} and~\ref{ass:sg}, let $\C$ be a $\compr$-compressor.
    Let $\alpha = 1$ when Assumption~\ref{ass:lipsg} holds and $\alpha=d$ otherwise.
    Then for Algorithm~\ref{alg:arbitrary_compressor} with $\reseterr=true$ and $\step = \tilde O \left(\min \left(\frac {\eps^2} {\alpha}, \frac {\compr \eps} \sqrtomc, \frac {\compr^2 \sqrt{\eps}} {(1 - \compr) d} \right)\right)$, after $T=\tilde O\left(\frac 1 {\eps^2 \step}\right) = \tilde O\left(\frac {\alpha} {\eps^4} + \frac {\sqrtomc} {\compr \eps^3} + \frac {d (1 - \compr)} {\compr^2 \eps^2 \sqrt \eps}\right)$ iterations, at least half of points $\vx_{t}$ such that the condition in Line~\ref{line:check_err_to_zero} is triggered at iteration $t$ are $\eps$-SOSP\@.
	The total additional communication due to Line~\ref{line:set_err_to_zero} is $\tilde O\left(\frac {d}{\eps^{\nicefrac 32}}\right)$.
\end{theorem}

In general, convergence to an $\eps$-SOSP is noticeably slower than convergence to an $\eps$-FOSP.
The reason for such behavior is that, in the analysis of second-order convergence, compression introduces an error similar to that of the stochastic noise.
When the stochastic gradient is Lipschitz (i.e. Assumption~\ref{ass:lipsg} holds), the number of iterations reduces by a factor of $d$.
Unfortunately, unlike the stochastic error, the compression is not Lipschitz even for deterministic gradients: e.g. consider a \topk compression applied to the vector where each coordinate is $1$ with small perturbation.
However, if the compressor is linear (Definition~\ref{def:linear_compr}), we show improved convergence rate:
the third term in the number of iterations decreases by the factor of $d$.

\begin{theorem}[Convergence to $\eps$-SOSP for linear compressor]
	\label{thm:sosp_randomk}
	Let $f$ satisfy Assumptions~\ref{ass:lip} and~\ref{ass:sg}, let $\C$ be a linear compressor.
    Let $\alpha = 1$ when Assumption~\ref{ass:lipsg} holds and $\alpha=d$ otherwise.
	Then for Algorithm~\ref{alg:arbitrary_compressor} with $\reseterr=false$ and $\step = \tilde O \left(\min \left(\frac {\eps^2} {\alpha}, \frac {\compr \eps} {\sqrtomc}, \frac {\compr^2 \sqrt{\eps}} {1 - \compr} \right)\right)$,
	after $T=\tilde O\left(\frac 1 {\eps^2 \step}\right) = \tilde O\left(\frac {\alpha} {\eps^4} + \frac {\sqrtomc} {\compr \eps^3} + \frac {1 - \compr} {\compr^2 \eps^2 \sqrt{\eps}}\right)$ iterations, at least half of visited points are $\eps$-SOSP\@.
\end{theorem}

\begin{corollary}
\label{cor:randomk_comm}
For \randomk compressor with $k = \frac {d \eps^{\nicefrac 34}} {\sqrt \alpha}$, the total number of iterations of Algorithm~\ref{alg:arbitrary_compressor} is $\tilde O(\frac \alpha {\eps^4})$ and the total communication per worker is $\tilde O\left(\frac {d \sqrt \alpha} {\eps^{3 + \nicefrac 14}}\right)$
\end{corollary}
\begin{proof}
    \randomk is a linear compressor requiring $O(k)$ communications.
    The total communication is $k$ times the number of iterations, i.e. $\tilde O\left(\frac {\alpha k} {\eps^4} + \frac {d} {\eps^3} + \frac {d^2} {k \eps^2 \sqrt \eps} \right)$.
    Balancing the first and the last term, we get $k^2 = \frac {d^2 \eps^{\nicefrac 32}} {\alpha}$.
    Substituting this value of $k$, we get the required result.
\end{proof}

Note that the total number of iterations matches the one for the uncompressed case (up to polylogarithmic factors), while the total communication decreases by a factor of $\nicefrac dk$:
\begin{corollary}
When Assumption~\ref{ass:lipsg} holds, the total communication for \randomk decreases by the factor of $\tilde \Theta(\eps^{-\nicefrac 34})$ compared with the unconstrained case.
Otherwise, the total communication decreases by the factor of $\tilde \Theta(\sqrt d \eps^{-\nicefrac 34})$.
\end{corollary}


\paragraph{Compressed SGD in Distributed Settings}

Below we consider different scenarios to illustrate how convergence depends on the properties of the compressor.
Recall that sketch-based \textsc{TopK} is a $\nicefrac kd$-compressor which requires $\tilde O(k)$ communication.
Selecting $\compr=\nicefrac kd$, with $k \ll d$, by Theorem~\ref{thm:sosp} we have
$\step = \tilde{O}\left(\min\left(\frac {\eps^2} {\alpha}, \frac {k \eps} {d}, \frac {k^2 \sqrt{\eps}} {d^3}\right)\right)$.
Therefore, the total number of iterations is $\tilde O\left(\frac {1} {\eps^4} + \frac{d}{k \eps^3} + \frac {d^3} {k^2 \eps^2 \sqrt{\eps}}\right)$
and the total communication is $\tilde O\left(\frac {k} {\eps^4} + \frac{d}{\eps^3} + \frac {d^3} {k \eps^2 \sqrt{\eps}}\right)$.

Note that the above reasoning considers a worst-case scenario. However, in practice it's often possible to achieve good compression at a low communication cost due to the fact that gradient coordinates have heavy-hitters, which are easy to recover using \topk. We formulate this beyond worst-case scenario as the following optional assumption:

\begin{assumption}[Optional]
\label{ass:good_sketch} There exists a constant $c < 1$ such that for all $t$, $\C(\nabla F(\vx_t, \sgdRand_t) + \ourNoise_t + \err_t)$ provides a $c$-compression and requires $\tilde O(1)$ bits of communication per worker. 
\end{assumption}

In other words, for all computed values, $\C$ provides a constant compression and requires a polylogarithmic amount of communication.
This assumption can be satisfied under various conditions. For example, some methods may take advantage of the situation when gradients between adjacent iterations are close~\citep{hanzely2018sega}.
In cases when certain coordinates are much more prominent in the gradient compared to others, \textsc{TopK} compressor will show good performance.

\begin{corollary}
\label{cor:table_comms}
Algorithm~\ref{alg:arbitrary_compressor} converges to $\eps$-SOSP in a number of settings, as shown in Table~\ref{tab:results}.

\begin{table*}[!t]
	\centering
	\caption{Convergence to $\eps$-SOSP with uncompressed SGD, with sketch-based \topk compressor, with \randomk compressor, and with a constant-compressor requiring constant communication (Assumption~\ref{ass:good_sketch}, beyond worst-case assumption).
	Each approach is considered in two settings: when Assumption~\ref{ass:lipsg} holds (i.e. the stochastic gradient is Lipschitz) and when it doesn't hold.
	For each approach we select an optimal compression factor based on our bounds. The results show that communication of SGD with \randomk compression outperforms that of the uncompressed SGD by $\tilde O(\eps^{-\nicefrac 14})$ when Assumption~\ref{ass:lipsg} holds and by $\tilde O(\sqrt d \eps^{-\nicefrac 34})$ otherwise. Based on our results, depending on $d$ and $\eps$, constant-memory compressor with constant communication may converge slower than \randomk, since such the compressor is not necessarily linear.
    }
	\label{tab:results}
	\begin{tabular}{cccccc}
		\hline
		\multicolumn{2}{c}{Setting}
		& $\compr$
        & Iterations & \multiline{Total comm.\\ per worker} & \multiline{Total comm.\\ improvement} \\
		\hline
		\hline
		\multirow{4}{*}[-1em]{\rotatebox[origin=c]{90}{Lipschitz $\nabla F$}}&
		\multiline{Uncompressed\\\citep{Fang2019sharp}} & $1$
			& $\tilde O\left(\frac 1 {\eps^{3.5}}\right)$
			& $\tilde O\left(\frac d {\eps^{3.5}}\right)$
			& $-$ \\
		\cline{2-6}
		&
		\multiline{Sketch-based\\ \topk} & \multiline{$\sqrt d \eps^{\nicefrac 34}$\\ ($\eps = o(d^{-\nicefrac 23})$)}
			& $\tilde O\left(\frac{1}{\eps^4}\right)$
			& $\tilde O\left(\frac{d \sqrt d }{\eps^{3 + \nicefrac 14}}\right)$
			& $\tilde \Theta\left(\frac {\eps^{-\nicefrac 14}} {\sqrt d}\right)$ \\
		\cline{2-6}
		&
		\multiline{\randomk} & $\eps^{\nicefrac 34}$
			& $\tilde O\left(\frac{1}{\eps^4}\right)$
			& $\tilde O\left(\frac{d }{\eps^{3 + \nicefrac 14}}\right)$
			& $\boldsymbol{\tilde \Theta\left(\eps^{-\nicefrac 14}\right)}$ \\
		\cline{2-6}
		&
		\multiline{Constant-memory\\$c$-compressor} & $c > 0$
			& $\tilde O\left(\frac 1 {\eps^4} + \frac d {\eps^2 \sqrt \eps} \right)$
			& $\tilde O\left(\frac 1 {\eps^4} + \frac d {\eps^2 \sqrt \eps}\right)$
			& $\tilde \Theta\left(\min(d, \frac 1 {\sqrt \eps}) \right)$ \\
		\hline
		\hline
		\multirow{4}{*}{\rotatebox[origin=c]{90}{non-Lipschitz $\nabla F$}}&
		\multiline{Uncompressed\\\citep{jin2019nonconvex}} & $1$
			& $\tilde O\left(\frac d {\eps^4}\right)$
			& $\tilde O\left(\frac {d^2} {\eps^4}\right)$
			& $-$ \\
		\cline{2-6}
		&
		\multiline{Sketch-based\\ \topk} & $\eps^{\nicefrac 34}$
			& $\tilde O\left(\frac{d}{\eps^4}\right)$
			& $\tilde O\left(\frac{d^2}{\eps^{3 + \nicefrac 14}}\right)$
			& $\tilde \Theta\left(\eps^{-\nicefrac 34}\right)$ \\
		\cline{2-6}
		&
		\multiline{\randomk} & $\frac {\eps^{\nicefrac 34}} {\sqrt d}$
			& $\tilde O\left(\frac{1}{\eps^4}\right)$
			& $\tilde O\left(\frac{d \sqrt d}{\eps^{3 + \nicefrac 14}}\right)$
			& $\boldsymbol{\tilde \Theta\left(\sqrt d \eps^{-\nicefrac 34}\right)}$ \\
		\cline{2-6}
		&
		\multiline{Constant-memory\\$c$-compressor} & $c > 0$
			& $\tilde O\left(\frac{d}{\eps^4}\right)$
			& $\tilde O\left(\frac{d}{\eps^4}\right)$
			& $\tilde \Theta\left(d \right)$\\
		\hline
	\end{tabular}
\end{table*}
\end{corollary}

\subsection{Proof Sketch}
\label{sec:proof_sketch}

In this section, we outline the main techniques used to prove Theorems~\ref{thm:fosp} and~\ref{thm:sosp}.
A recent breakthrough line of work focused on convergence of first-order methods to $\eps$-SOSP~\cite{geHJY15,CarmonD16gradient,JinGNKJ17,tripuraneni2018stochastic,jin2019nonconvex} (JACM) has developed a comprehensive set of analytic techniques.
We start by outlining~\citet{jin2019nonconvex}, which is the sharpest known SGD analysis in the case when the stochastic gradient is not Lipschitz.

Let $\vx_0$ be an iterate such that $\lmin(\nabla^2 f(\vx_0)) < -\sqrtre$, and $\vv_1$ be the eigenvector corresponding to $\lmin$.
Consider sequences $\{\vx_t\}$ and $\{\vx'_t\}$ starting with $\vx_0$ which are referred to as \emph{coupling sequences}: their distributions match the distribution of compressed SGD iterates (i.e. both sequences can be produced by Algorithm~\ref{alg:arbitrary_compressor}), and they share the same randomness, with an exception that their artificial noise has the opposite sign in the direction $\vv_1$.
The main idea is that such artificial noise combined with SGD updates ensures that projection of $\vx_t - \vx'_t$ on $\vv_1$ increases exponentially, and therefore at least one of the sequences moves far from $\vx$.
After that, one can use an ``Improve or localize'' Lemma which states that, if we move far from the original point, then the objective decreases substantially.

If we have an access to a deterministic gradient oracle and the objective function is quadratic, then gradient descent behaves similarly to the power method, since in this case:
\[\vx_{t+1} = \vx_t - \step \nabla f(\vx_t) = \vx_t - \step \nabla^2 f(\vx_0) \vx_t = (I - \step \nabla^2 f(\vx_0)) \vx_t\]
Adding artificial noise guarantees that projection of $\vx_t -\vx_t'$ on direction $\vv_1$ is large, and the power method further amplifies this projection.

In general, the SGD behavior deviates from power method due to: 1) the difference between $f$ and its quadratic approximation and 2) stochastic noise.
\citet{jin2019nonconvex} show that the errors introduced by these deviations are dominated by the increase in direction $\vv_1$, and therefore SGD successfully escapes saddle points.

\subsubsection{Outline of our compressed SGD analysis.} 

The analysis above is not applicable to our algorithm due to gradient compression and error-feedback.
Moreover, in the case of an arbitrary compressor we change the algorithm even further by periodically setting $\err_t$ to $0$.



One of the major changes is that errors introduced by the compression lead to even greater deviation of SGD from the power method, and this deviation can potentially dominate other terms: if the compression error is accumulated from the beginning of the algorithm execution, then the compression error can be arbitrarily large.
Surprisingly, we show that, for a linear compressor, such adverse behavior doesn't happen.
Let $\err'_t$ be the compression error sequence corresponding to $\vx'_t$ such that $\err'_0 = \err_0$.
Then the deviation of SGD from the power method caused by compression can be quantified as
\[\mathcal E_t = \step^2 \nabla^2 f(\vx_0) \sum_{i=1}^{t-1}(I - \step \nabla^2 f(\vx_0))^{t-1-i} {\color{blue}(\err_i - \err'_i)}, \quad\text{(Proposition~\ref{prop:diff_decomposition})}\]
and therefore, we have to bound $\|\err_i - \err_i'\|$ for all $i$.
For $G_t = \err_t + \nabla F(\vx_t, \sgdRand_t) + \ourNoise_t$ (with $G_t'$ defined analogously):
\begin{align*}
    \Exp{\|\err_{t+1} - \err_{t+1}'\|^2}
    &= \Exp{\|(G_t - \C(G_t)) - (G_t' - \C(G_t'))\|^2} \\
    &= \Exp{\|(G_t - G_t') - \C(G_t - G_t')\|^2} \\
    &\le (1 - \compr) \Exp{\|G_t - G_t'\|^2} \\
    &= (1 - \compr) \Exp{\|(\err_t - \err_t') + (\nabla F(\vx_t, \sgdRand_t) - \nabla F(\vx'_t, \sgdRand_t)) + (\ourNoise_t - \ourNoise_t')\|^2}.
\end{align*}
Since $\err_{0} = \err'_{0}$, after telescoping, $\|\err_{t+1} - \err_{t+1}'\|$ can be bounded using $\|\nabla F(\vx_i, \sgdRand_i) - \nabla F(\vx'_i, \sgdRand_i)\|$ and $\|\ourNoise_i - \ourNoise'_i\|$ for $i \in [0: t]$ (Lemma~\ref{lem:bound_error_sum_linear}).
In other words, when escaping from a saddle point, $\mathcal E_t$ can bounded based on gradients and noises encountered during escaping. 
Therefore it is comparable to other terms and can be bounded with an appropriate choice of $\step$.

Unfortunately, for the arbitrary compressor case we don't have a good estimation on $\mathcal E_t$, since in general we don't have better bound on $\|\err_i - \err'_i\|$ than $\|\err_i\| + \|\err'_i\|$ (see proof of Lemma~\ref{lem:bound_error_sum}).
Lemma~\ref{lem:error_estimation} bounds the compression error $\err_t$ in terms of $\|\nabla f(\vx_0)\|, \ldots, \|\nabla f(\vx_t)\|$, analogously to the derivation above:
\[\Exp{\|\err_{t}\|^2} \le \frac {2 (1 - \compr)} {\compr} \sum_{i=0}^{t-1} \left(1 - \frac \compr 2\right)^{t-i} \Exp{{\color{blue}\|\nabla f(\vx_i)\|^2} + \fullStd^2},\]
but the bound depends on all gradients starting from the first iteration.
To solve this problem, we periodically set the compression error to $0$ (Line~\ref{line:check_err_to_zero} of Algorithm~\ref{alg:arbitrary_compressor}).
Let $t_0$ be an iteration such that $\err_{t_0}$ is set to $0$: then, when escaping from $\vx_{t_0}$, we can apply Lemma~\ref{lem:error_estimation} with $i$ starting from $t_0$.
This leads to major difference from the \cite{jin2019nonconvex} analysis: we need to consider large- and small-gradient cases separately.
When the gradient at $\vx_{t_0}$ is large (Lemma~\ref{lem:large_gradient}), we show that nearby gradients are also large, and the objective improves by the Compressed Descent Lemma~\ref{lem:descent_lemma}.
Otherwise, we can bound the error norm for the next few iterations (Lemma~\ref{lem:bound_error_sum}).


Finally, the analysis uses not only the sequence of iterates $\{\vx_t\}$, but also the corrected sequence $\{\tvx_t\}$ where $\tvx_t = \vx_t - \step \err_t$ (similarly, $\tvx'_t = \vx'_t - \step \err'_t$).
Intuitively, $\err_t$ accumulates the difference between the communicated and the original gradient, and therefore the goal of $\tvx_t$ is to offset the compression error.
Typically, $\vx_t$ is used as an argument of $\nabla f(\cdot)$, while $\tvx_t$ is used in distances and as an argument of $f(\cdot)$, which noticeably complicates the analysis.
In particular, if some property holds for $\vx_t$, it doesn't necessarily hold for $\tvx_t$ and vice versa: for example, since $\vx_t$ and $\tvx_t$ are not necessarily close, bound $\|\tvx_t - \tvx'_t\|$ doesn't in general imply bound on $\|\vx_t - \vx'_t\|$.
However, in our analysis, we show that we can bound $\|\vx_t - \vx'_t\|$, which is required to bound $\|\nabla f(\vx_t) - \nabla f(\vx'_t)\|$ in Lemma~\ref{lem:nonlocalize}.
\section{Experiments}

\begin{figure*}[p]
    \centering
    \begin{subfigure}[t]{0.99\textwidth}
        \includegraphics[width=\textwidth]{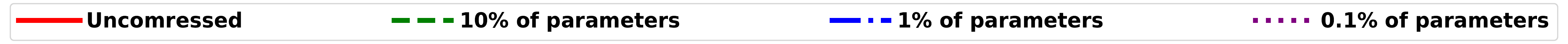}
    \end{subfigure}

    \begin{subfigure}[t]{0.32\textwidth}
        \includegraphics[width=\textwidth]{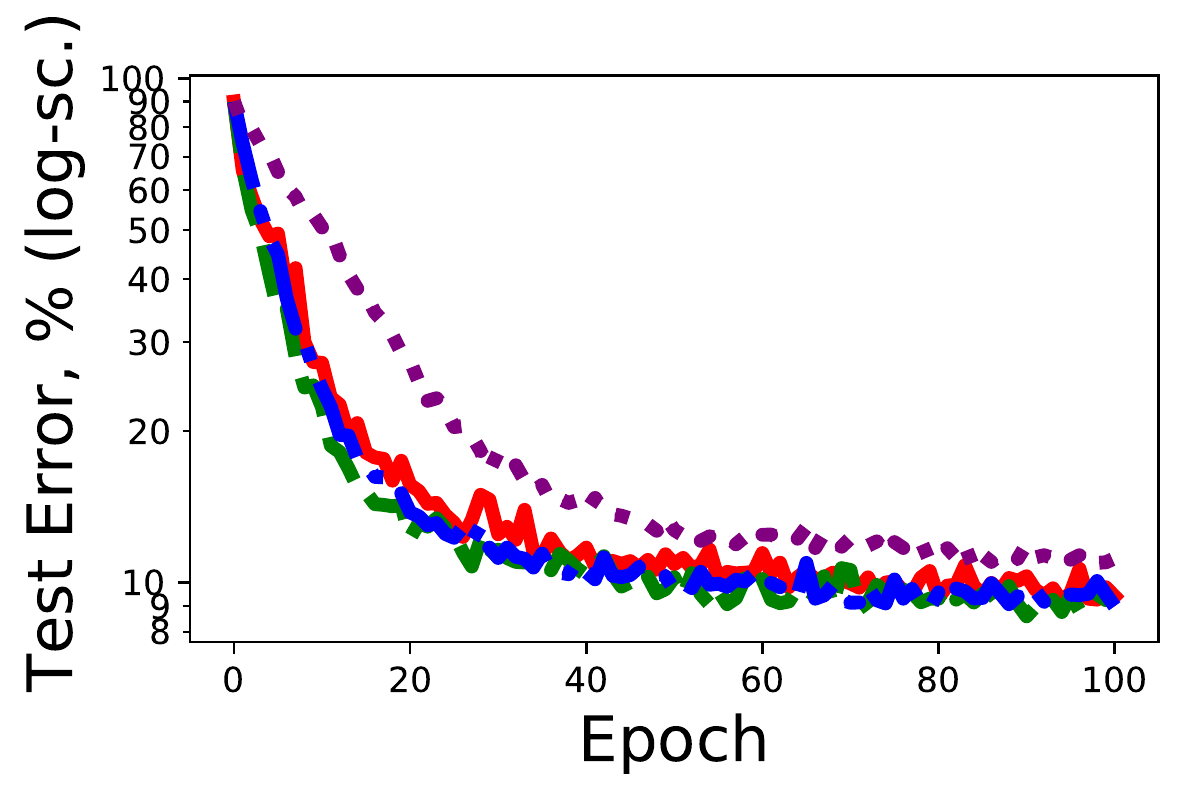}
        \caption{Test Error}
        \label{fig:test_acc}
    \end{subfigure}
    \begin{subfigure}[t]{0.32\textwidth}
        \includegraphics[width=\textwidth]{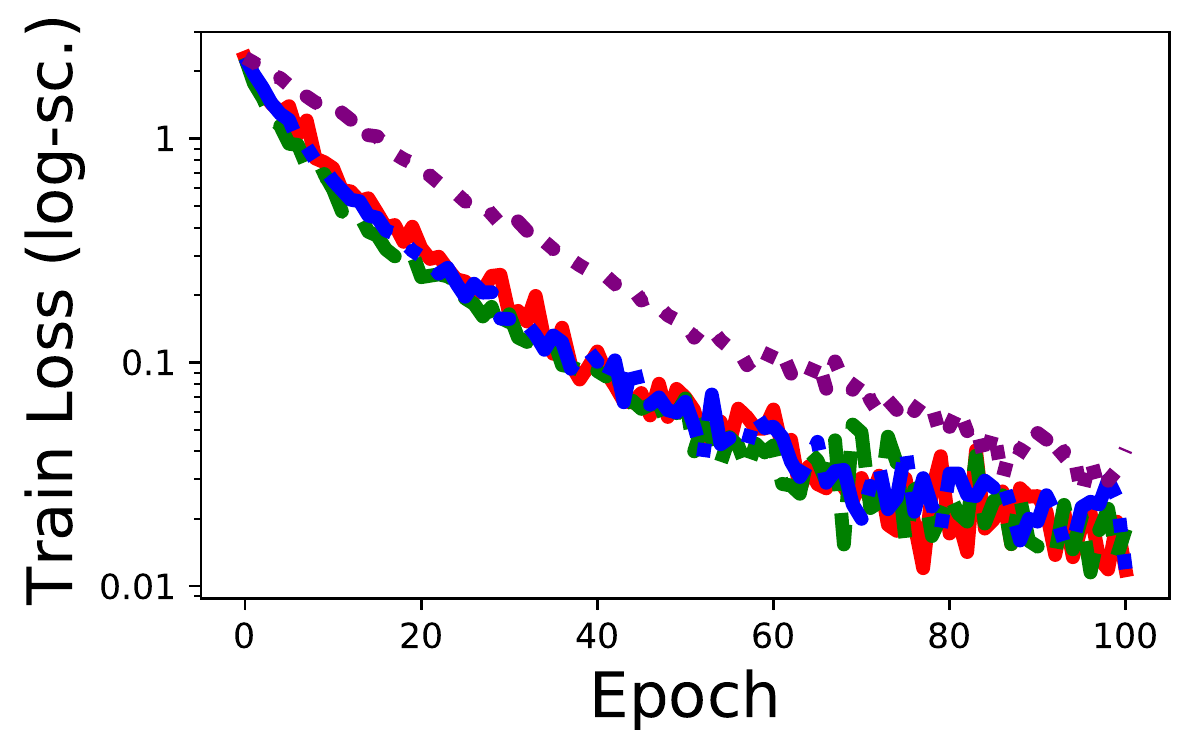}
        \caption{Train Loss}
        \label{fig:train_loss}
    \end{subfigure}
    \begin{subfigure}[t]{0.32\textwidth}
        \includegraphics[width=\textwidth]{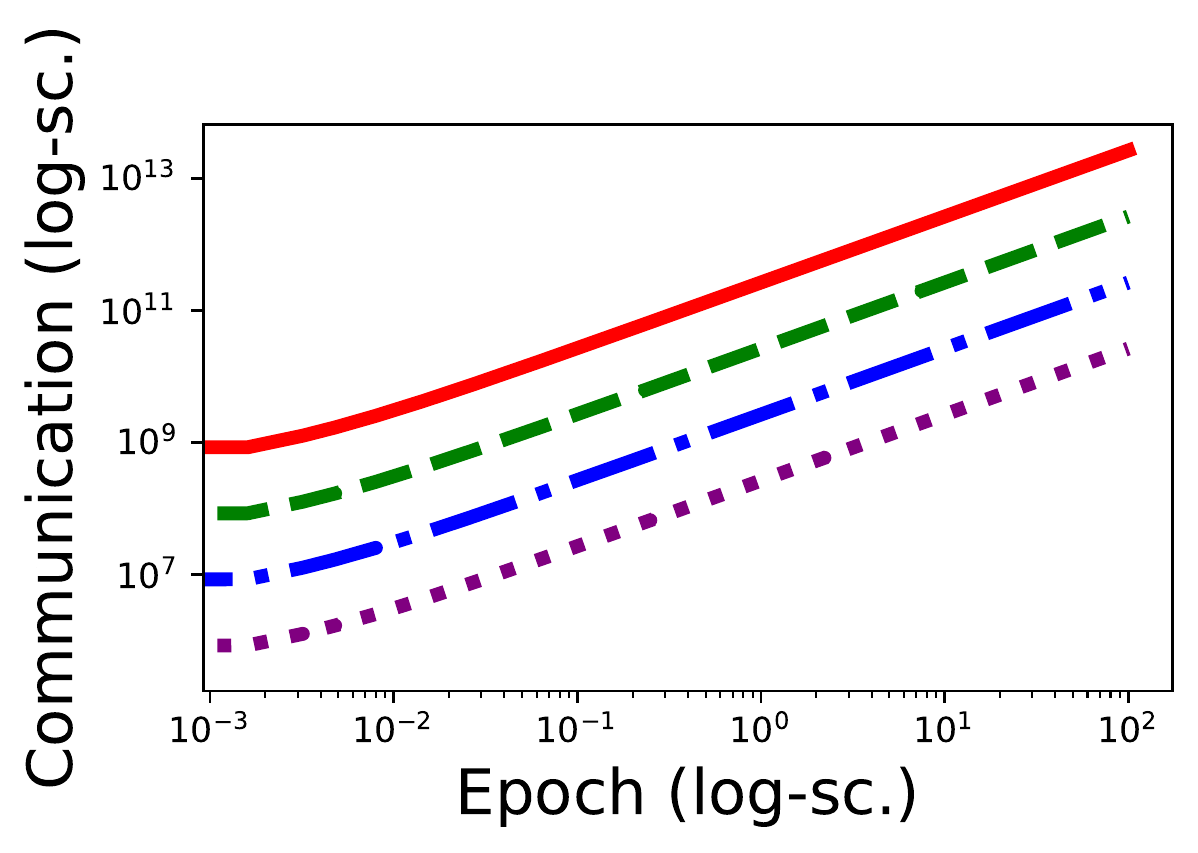}
        \caption{Total number of communicated parameters}
        \label{fig:communication}
    \end{subfigure}
    \caption{Convergence of distributed SGD ($\step=0.1$, batch size is $8$ per machine) with \randomk compressor when $100\%$ (full gradient), $1\%$, $0.1\%$ and $0.01\%$ of coordinates are used.
             ResNet34 model is trained on CIFAR-10 distributed across $10$ machines: each machine corresponds to a single class.
             SGD with $10\%$ and $1\%$ compression achieves performance similar to that of uncompressed SGD, while requiring significantly less communication}
    \label{fig:experiments}
\end{figure*}
\begin{figure*}[p]
    {\centering
    \begin{subfigure}[t]{0.6\textwidth}
        \includegraphics[width=\textwidth]{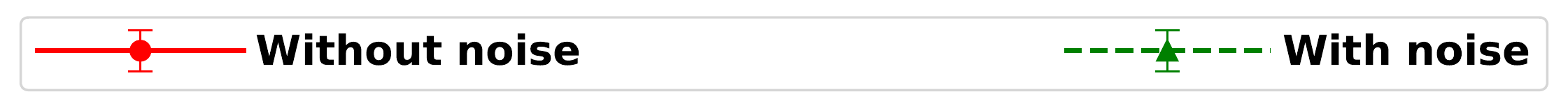}
    \end{subfigure}

    \begin{subfigure}[t]{0.32\textwidth}
        \includegraphics[width=\textwidth]{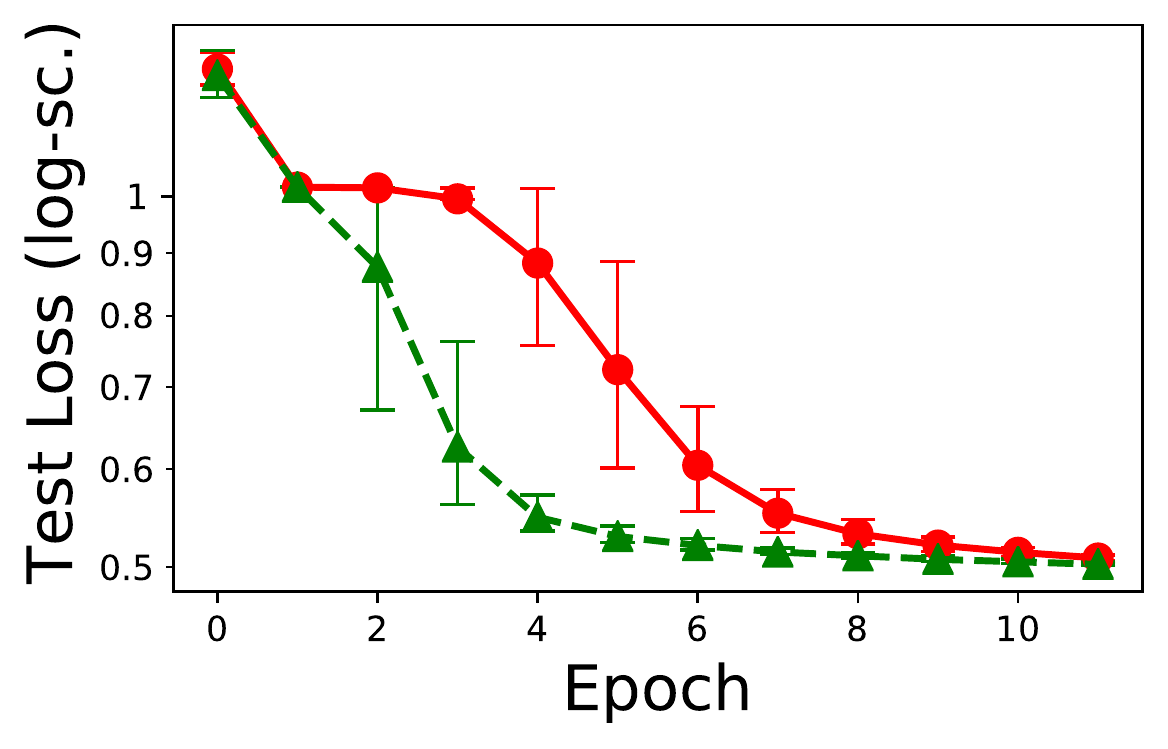}
        \caption{Uncompressed loss}
        \label{fig:uncompressed_loss}
    \end{subfigure}
    \begin{subfigure}[t]{0.32\textwidth}
        \includegraphics[width=\textwidth]{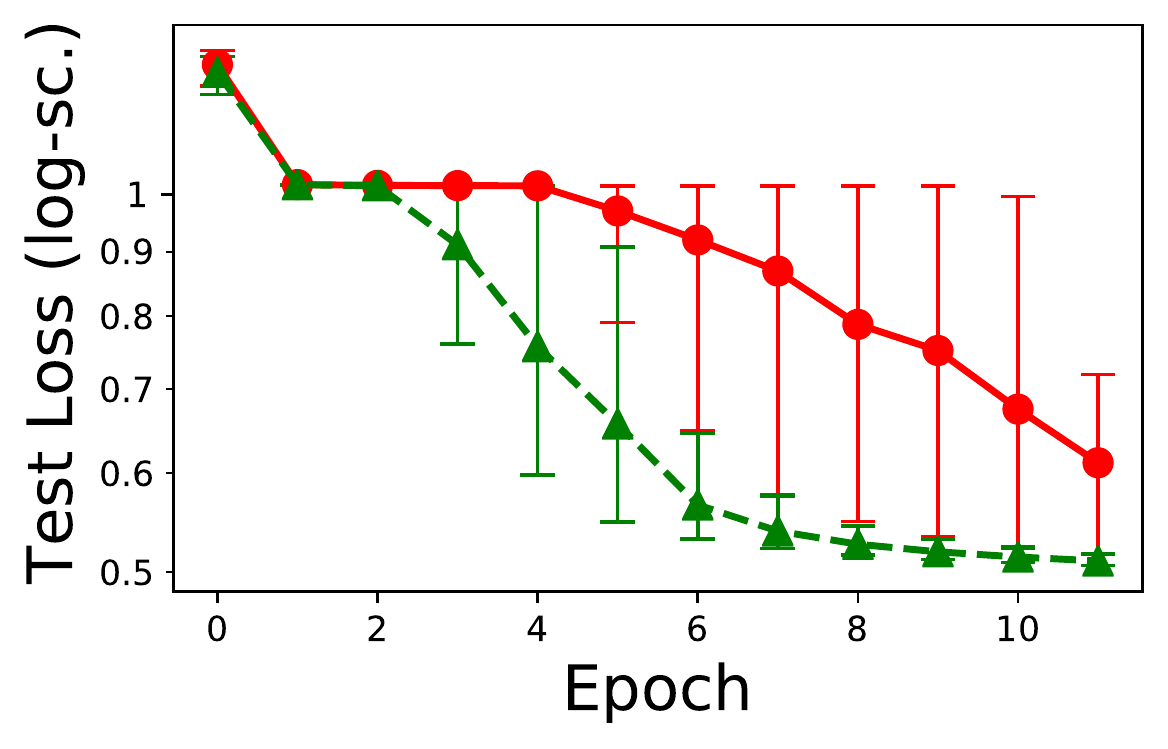}
        \caption{TopK loss}
    \end{subfigure}
    \begin{subfigure}[t]{0.32\textwidth}
        \includegraphics[width=\textwidth]{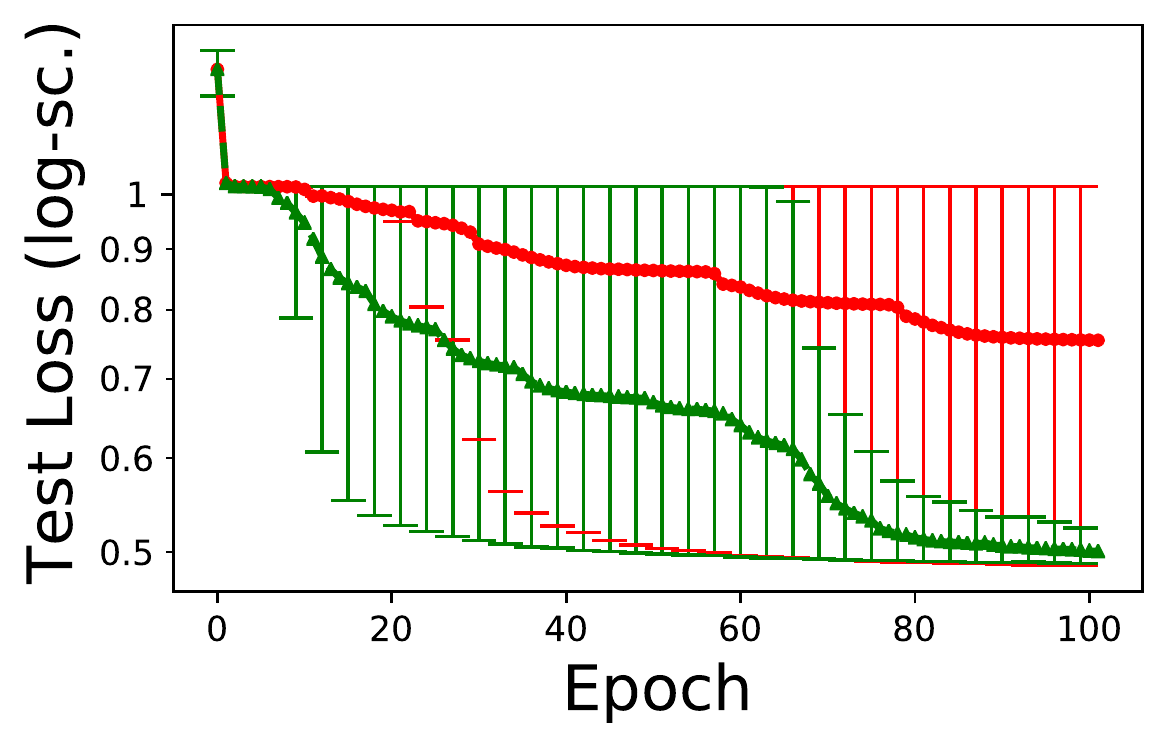}
        \caption{Random compressor loss}
    \end{subfigure}

    \begin{subfigure}[t]{0.32\textwidth}
        \includegraphics[width=\textwidth]{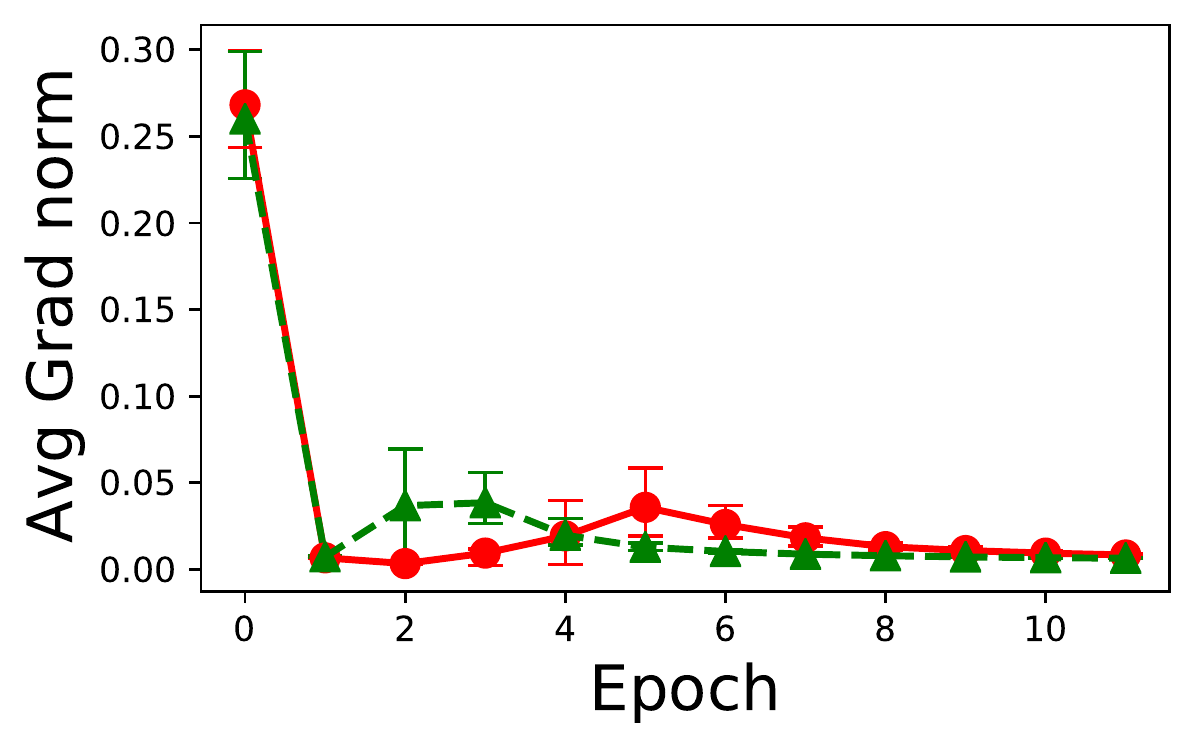}
        \caption{Uncompressed gradient norm}
    \end{subfigure}
    \begin{subfigure}[t]{0.32\textwidth}
        \includegraphics[width=\textwidth]{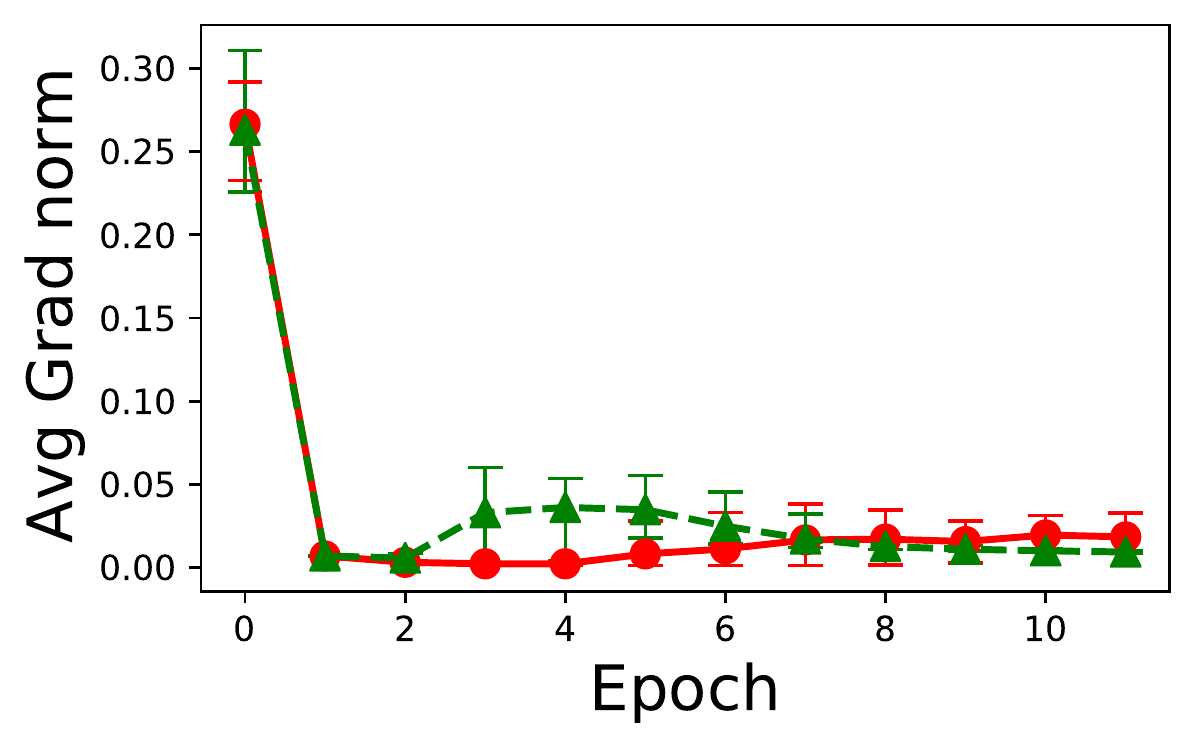}
        \caption{TopK gradient norm}
    \end{subfigure}
    \begin{subfigure}[t]{0.32\textwidth}
        \includegraphics[width=\textwidth]{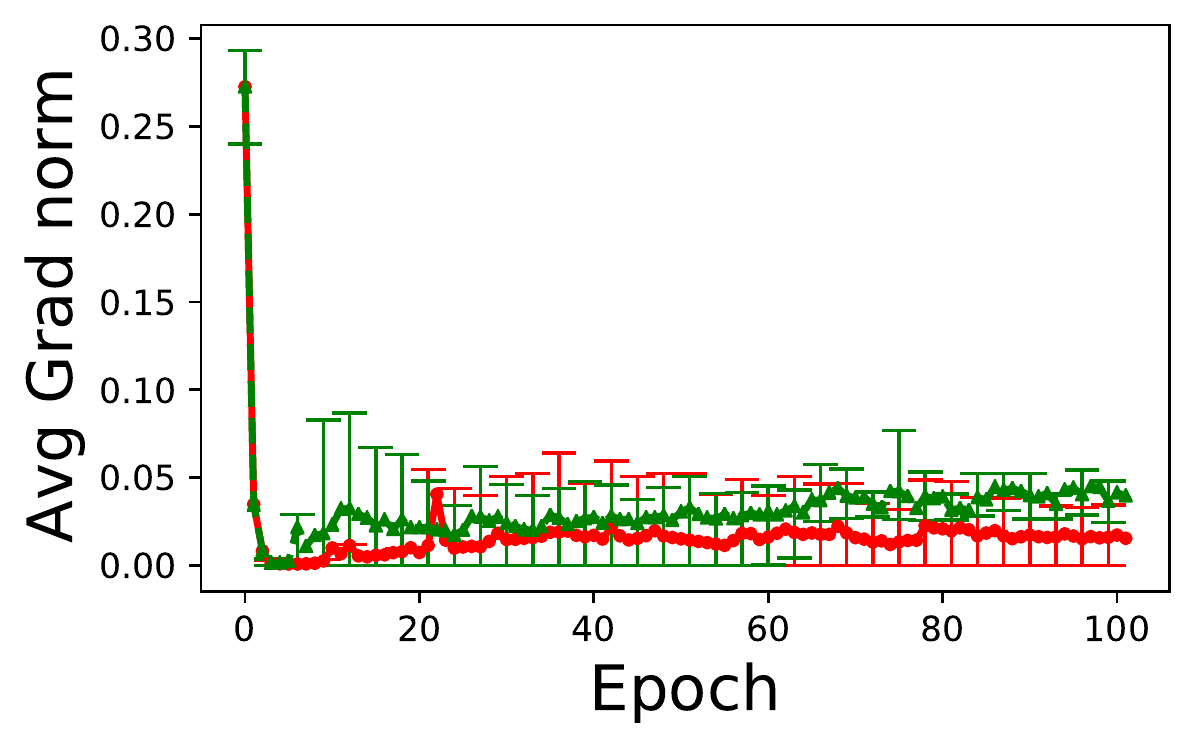}
        \caption{\randomk gradient norm}
    \end{subfigure}

    \caption[CAP]{
            Convergence of SGD ($\step=0.1$, batch size is $64$) without compression (left), with \topk($0.1\%$ of coordinates) compression (middle), and with \randomk($0.1\%$ of coordinates) compression (right) on MNIST autoencoder dataset without noise (red) and with Gaussian noise (green, $\sigma=0.01$ for each coordinate).
             Data points correspond to average over $10$ executions and error bars correspond to $10\%$- and $90\%$-quantiles.
             The bottom row shows the norms of the gradients averaged over the last $100$ iterations.
             The figure shows that SGD encounters and escapes saddle points for all compressors, and adding noise facilitates escaping from the saddle points${}^\dagger$.
            }
    \label{fig:saddle_experiments}}
    {\hrule\footnotesize ${}^\dagger$ For the sake of presentation, to ensure that gradient converges to $0$, we decrease the magnitude of the artificial noise at later iterations. With a fixed noise magnitude, as our theory predicts, gradient norm converges if a smaller step size is used, but this requires significantly more iterations, and the presentation is less clear. Note that this modification only affects the gradient convergence as the objective converges even with fixed noise and a large step size.}
\end{figure*}

In our experiments, we show that noisy Compressed SGD achieves convergence comparable with full SGD and successfully escapes saddle points.
We perform our first set of experiments on ResNet34 model trained using CIFAR-10 dataset with step size $0.1$.
We distribute the data across $10$ machines, such that each machine contains data from a single class.
We analyze convergence of compressed SGD with \randomk compressor when $100\%$, $10\%$, $1\%$ and $0.1\%$ random gradient coordinates are communicated.
Figure~\ref{fig:experiments} shows that SGD with \randomk with $10\%$ or $1\%$ of coordinates compression converges as fast as the full SGD, while requiring substantially smaller communication.

In our second set of experiments, we show that SGD indeed encounters saddle points and noise facilitates escaping from them.
We compare uncompressed SGD, SGD with \topk compressor ($0.1\%$ of coordinates), and SGD with \randomk compressor ($0.1\%$ of coordinates) on deep MNIST autoencoder.
The encoder is defined using $3$ convolutional layers with ReLU activation, with the following parameters (the decoder is symmetrical): (channels=16, kernel=3, stride=2, padding=1), (channels=32, kernel=3, stride=2, padding=1) and (channels=64, kernel=7, stride=1, padding=0).
For all settings, we compare their convergence rates with and without noise.
Figure~\ref{fig:saddle_experiments} shows that SGD does encounter saddle points: e.g. in Figure~\ref{fig:uncompressed_loss}, for SGD without noise, during epochs $1$-$3$, the gradient norm is close to $0$ and the objective value doesn't improve.
However, compressed SGD escapes from the saddle points, and noise significantly improves the escaping rate.

\section{Conclusion}

This paper shows the first result for convergence of compressed SGD to an $\eps$-SOSP, and it's possible that the convergence rate can be further improved.
In particular, it's unclear whether $\sqrt \eps$ in the last terms in Theorems~\ref{thm:sosp} and~\ref{thm:sosp_randomk} is required or it is an artifact of the analysis.
Moreover, when Assumption~\ref{ass:lipsg} holds, depending on $\eps$ and $d$, the convergence rate of compressed SGD with arbitrary compressor may be worse compared with uncompressed SGD, and it's not clear whether it's actually the case.

When Assumption~\ref{ass:lipsg} holds, the communication can probably be improved by the factor of $\eps^{-\nicefrac 14}$ using techniques from~\citet{Fang2019sharp}, which achieve $\tilde O(\eps^{-3.5})$ convergence rate under Assumption~\ref{ass:lipsg} outperforming $\tilde O(\eps^{-4})$ from~\cite{jin2019nonconvex} by the factor of $\eps^{-\nicefrac 12}$.
When balancing the terms in Theorems~\ref{thm:sosp} and~\ref{thm:sosp_randomk}, the communication improvement will be the square root of this value.
Similarly, using a variance reduction techniques (which achieve $\tilde O(\eps^{-3})$ convergence rate), one may expect $\eps^{-\nicefrac 12}$ improvement.

Finally, it's unclear whether linearity of the compressor is required for the bound from Theorem~\ref{thm:sosp_randomk}.
We suspect that the assumption can be relaxed: similarly to the stochastic gradient case, it may suffice for the compressor to be Lipschitz.
Proving this and showing (theoretically or empirically) that this property holds for the existing compressors is another interesting future direction.

\arxiv{\bibliographystyle{plainnat}}
\bibliography{references}

\begin{thebibliography}{44}
\providecommand{\natexlab}[1]{#1}
\providecommand{\url}[1]{\texttt{#1}}
\expandafter\ifx\csname urlstyle\endcsname\relax
  \providecommand{\doi}[1]{doi: #1}\else
  \providecommand{\doi}{doi: \begingroup \urlstyle{rm}\Url}\fi

\bibitem[Agarwal et~al.(2017)Agarwal, Allen-Zhu, Bullins, Hazan, and
  Ma]{agarwal2017finding}
Naman Agarwal, Zeyuan Allen-Zhu, Brian Bullins, Elad Hazan, and Tengyu Ma.
\newblock Finding approximate local minima faster than gradient descent.
\newblock In \emph{Proceedings of the 49th Annual ACM SIGACT Symposium on
  Theory of Computing}, pages 1195--1199. ACM, 2017.

\bibitem[Alistarh et~al.(2017)Alistarh, Grubic, Li, Tomioka, and
  Vojnovic]{alistarh2017qsgd}
Dan Alistarh, Demjan Grubic, Jerry Li, Ryota Tomioka, and Milan Vojnovic.
\newblock Qsgd: Communication-efficient sgd via gradient quantization and
  encoding.
\newblock In \emph{Advances in Neural Information Processing Systems}, pages
  1709--1720, 2017.

\bibitem[Alistarh et~al.(2018)Alistarh, Hoefler, Johansson, Khirirat,
  Konstantinov, and Renggli]{alistarh2018convergence}
Dan Alistarh, Torsten Hoefler, Mikael Johansson, Sarit Khirirat, Nikola
  Konstantinov, and C{\'e}dric Renggli.
\newblock The convergence of sparsified gradient methods.
\newblock \emph{arXiv preprint arXiv:1809.10505}, 2018.

\bibitem[Allen-Zhu(2018)]{allen2018natasha}
Zeyuan Allen-Zhu.
\newblock Natasha 2: Faster non-convex optimization than sgd.
\newblock In \emph{Advances in neural information processing systems}, pages
  2675--2686, 2018.

\bibitem[Allen-Zhu and Li(2018)]{allen2018neon2}
Zeyuan Allen-Zhu and Yuanzhi Li.
\newblock Neon2: Finding local minima via first-order oracles.
\newblock In \emph{Advances in Neural Information Processing Systems}, pages
  3716--3726, 2018.

\bibitem[Anandkumar and Ge(2016)]{anandkumarG16local}
Animashree Anandkumar and Rong Ge.
\newblock Efficient approaches for escaping higher order saddle points in
  non-convex optimization.
\newblock In Vitaly Feldman, Alexander Rakhlin, and Ohad Shamir, editors,
  \emph{29th Annual Conference on Learning Theory}, volume~49 of
  \emph{Proceedings of Machine Learning Research}, pages 81--102, Columbia
  University, New York, New York, USA, 23--26 Jun 2016. PMLR.
\newblock URL \url{http://proceedings.mlr.press/v49/anandkumar16.html}.

\bibitem[Bernstein et~al.(2018)Bernstein, Wang, Azizzadenesheli, and
  Anandkumar]{bernstein2018signsgd}
Jeremy Bernstein, Yu-Xiang Wang, Kamyar Azizzadenesheli, and Anima Anandkumar.
\newblock signsgd: Compressed optimisation for non-convex problems.
\newblock \emph{arXiv preprint arXiv:1802.04434}, 2018.

\bibitem[Bhojanapalli et~al.(2016)Bhojanapalli, Neyshabur, and
  Srebro]{bhojanapalliNS16matrix}
Srinadh Bhojanapalli, Behnam Neyshabur, and Nati Srebro.
\newblock Global optimality of local search for low rank matrix recovery.
\newblock In D.~D. Lee, M.~Sugiyama, U.~V. Luxburg, I.~Guyon, and R.~Garnett,
  editors, \emph{Advances in Neural Information Processing Systems 29}, pages
  3873--3881. Curran Associates, Inc., 2016.
\newblock URL
  \url{http://papers.nips.cc/paper/6271-global-optimality-of-local-search-for-low-rank-matrix-recovery.pdf}.

\bibitem[Boyd et~al.(2004)Boyd, Boyd, and Vandenberghe]{boyd2004convex}
Stephen Boyd, Stephen~P Boyd, and Lieven Vandenberghe.
\newblock \emph{Convex optimization}.
\newblock Cambridge university press, 2004.

\bibitem[Carmon and Duchi(2016)]{CarmonD16gradient}
Yair Carmon and John~C. Duchi.
\newblock Gradient descent efficiently finds the cubic-regularized non-convex
  newton step.
\newblock \emph{ArXiv}, abs/1612.00547, 2016.

\bibitem[Carmon et~al.(2016)Carmon, Duchi, Hinder, and
  Sidford]{carmon2016accelerated}
Yair Carmon, John~C Duchi, Oliver Hinder, and Aaron Sidford.
\newblock Accelerated methods for non-convex optimization.
\newblock \emph{arXiv preprint arXiv:1611.00756}, 2016.

\bibitem[Carmon et~al.(2017)Carmon, Hinder, Duchi, and
  Sidford]{carmon2017convex}
Yair Carmon, Oliver Hinder, John~C Duchi, and Aaron Sidford.
\newblock Convex until proven guilty: Dimension-free acceleration of gradient
  descent on non-convex functions.
\newblock \emph{arXiv preprint arXiv:1705.02766}, 2017.

\bibitem[Chilimbi et~al.(2014)Chilimbi, Suzue, Apacible, and
  Kalyanaraman]{chilimbi2014project}
Trishul Chilimbi, Yutaka Suzue, Johnson Apacible, and Karthik Kalyanaraman.
\newblock Project adam: Building an efficient and scalable deep learning
  training system.
\newblock In \emph{11th USENIX Symposium on Operating Systems Design and
  Implementation (OSDI 14)}, pages 571--582, 2014.

\bibitem[Choromanska et~al.(2015)Choromanska, Henaff, Mathieu, Arous, and
  LeCun]{choromanska2015loss}
Anna Choromanska, Mikael Henaff, Michael Mathieu, G{\'e}rard~Ben Arous, and
  Yann LeCun.
\newblock The loss surfaces of multilayer networks.
\newblock In \emph{Artificial intelligence and statistics}, pages 192--204.
  PMLR, 2015.

\bibitem[Curtis et~al.(2014)Curtis, Robinson, and Samadi]{curtis2014trust}
Frank~E Curtis, Daniel~P Robinson, and Mohammadreza Samadi.
\newblock A trust region algorithm with a worst-case iteration complexity of
  $o(\epsilon^ {-3/2} )$ for nonconvex optimization.
\newblock \emph{Mathematical Programming}, pages 1--32, 2014.

\bibitem[Dauphin et~al.(2014)Dauphin, Pascanu, Gulcehre, Cho, Ganguli, and
  Bengio]{dauphin2014identifying}
Yann~N Dauphin, Razvan Pascanu, Caglar Gulcehre, Kyunghyun Cho, Surya Ganguli,
  and Yoshua Bengio.
\newblock Identifying and attacking the saddle point problem in
  high-dimensional non-convex optimization.
\newblock \emph{Advances in Neural Information Processing Systems},
  27:\penalty0 2933--2941, 2014.

\bibitem[Dekel et~al.(2012)Dekel, Gilad-Bachrach, Shamir, and
  Xiao]{dekel2012optimal}
Ofer Dekel, Ran Gilad-Bachrach, Ohad Shamir, and Lin Xiao.
\newblock Optimal distributed online prediction using mini-batches.
\newblock \emph{Journal of Machine Learning Research}, 13\penalty0 (1), 2012.

\bibitem[Du et~al.(2017)Du, Jin, Lee, Jordan, Singh, and
  Poczos]{du2017gradient}
Simon~S Du, Chi Jin, Jason~D Lee, Michael~I Jordan, Aarti Singh, and Barnabas
  Poczos.
\newblock Gradient descent can take exponential time to escape saddle points.
\newblock In \emph{Advances in neural information processing systems}, pages
  1067--1077, 2017.

\bibitem[Fang et~al.(2018)Fang, Li, Lin, and Zhang]{fang2018spider}
Cong Fang, Chris~Junchi Li, Zhouchen Lin, and Tong Zhang.
\newblock Spider: Near-optimal non-convex optimization via stochastic
  path-integrated differential estimator.
\newblock In \emph{Advances in Neural Information Processing Systems}, pages
  687--697, 2018.

\bibitem[Fang et~al.(2019)Fang, Lin, and Zhang]{Fang2019sharp}
Cong Fang, Zhouchen Lin, and Tong Zhang.
\newblock Sharp analysis for nonconvex sgd escaping from saddle points.
\newblock \emph{arXiv preprint arXiv:1902.00247}, 2019.

\bibitem[Ge et~al.(2015)Ge, Huang, Jin, and Yuan]{geHJY15}
Rong Ge, Furong Huang, Chi Jin, and Yang Yuan.
\newblock Escaping from saddle points—online stochastic gradient for tensor
  decomposition.
\newblock In \emph{Conference on Learning Theory}, pages 797--842, 2015.

\bibitem[Ge et~al.(2016)Ge, Lee, and Ma]{geLM16matrix}
Rong Ge, Jason~D. Lee, and Tengyu Ma.
\newblock Matrix completion has no spurious local minimum.
\newblock In \emph{Proceedings of the 30th International Conference on Neural
  Information Processing Systems}, NIPS’16, page 2981–2989, Red Hook, NY,
  USA, 2016. Curran Associates Inc.
\newblock ISBN 9781510838819.

\bibitem[Hanzely et~al.(2018)Hanzely, Mishchenko, and
  Richt{\'a}rik]{hanzely2018sega}
Filip Hanzely, Konstantin Mishchenko, and Peter Richt{\'a}rik.
\newblock Sega: Variance reduction via gradient sketching.
\newblock In \emph{Advances in Neural Information Processing Systems}, pages
  2082--2093, 2018.

\bibitem[Ivkin et~al.(2019)Ivkin, Rothchild, Ullah, Stoica, Arora,
  et~al.]{ivkin2019communication}
Nikita Ivkin, Daniel Rothchild, Enayat Ullah, Ion Stoica, Raman Arora, et~al.
\newblock Communication-efficient distributed sgd with sketching.
\newblock In \emph{Advances in Neural Information Processing Systems}, pages
  13144--13154, 2019.

\bibitem[Jin et~al.(2017)Jin, Ge, Netrapalli, Kakade, and Jordan]{JinGNKJ17}
Chi Jin, Rong Ge, Praneeth Netrapalli, Sham~M. Kakade, and Michael~I. Jordan.
\newblock How to escape saddle points efficiently.
\newblock In \emph{Proceedings of the 34th International Conference on Machine
  Learning - Volume 70}, ICML'17, pages 1724--1732. JMLR.org, 2017.
\newblock URL \url{http://dl.acm.org/citation.cfm?id=3305381.3305559}.

\bibitem[Jin et~al.(2018)Jin, Netrapalli, and Jordan]{jin2018accelerated}
Chi Jin, Praneeth Netrapalli, and Michael~I Jordan.
\newblock Accelerated gradient descent escapes saddle points faster than
  gradient descent.
\newblock In \emph{Conference On Learning Theory}, pages 1042--1085. PMLR,
  2018.

\bibitem[Jin et~al.(2021)Jin, Netrapalli, Ge, Kakade, and
  Jordan]{jin2019nonconvex}
Chi Jin, Praneeth Netrapalli, Rong Ge, Sham~M Kakade, and Michael~I Jordan.
\newblock On nonconvex optimization for machine learning: Gradients,
  stochasticity, and saddle points.
\newblock \emph{Journal of the ACM (JACM)}, 68\penalty0 (2):\penalty0 1--29,
  2021.

\bibitem[Karimireddy et~al.(2019)Karimireddy, Rebjock, Stich, and
  Jaggi]{karimireddy2019error}
Sai~Praneeth Karimireddy, Quentin Rebjock, Sebastian~U Stich, and Martin Jaggi.
\newblock Error feedback fixes signsgd and other gradient compression schemes.
\newblock \emph{arXiv preprint arXiv:1901.09847}, 2019.

\bibitem[Lee et~al.(2016)Lee, Simchowitz, Jordan, and Recht]{lee2016gradient}
Jason~D Lee, Max Simchowitz, Michael~I Jordan, and Benjamin Recht.
\newblock Gradient descent only converges to minimizers.
\newblock In \emph{Conference on learning theory}, pages 1246--1257, 2016.

\bibitem[Li et~al.(2014)Li, Andersen, Smola, and Yu]{li2014communication}
Mu~Li, David~G Andersen, Alexander~J Smola, and Kai Yu.
\newblock Communication efficient distributed machine learning with the
  parameter server.
\newblock In \emph{NIPS}, volume~2, pages 1--4, 2014.

\bibitem[McMahan et~al.(2017)McMahan, Moore, Ramage, Hampson, and
  y~Arcas]{mcmahan2017communication}
Brendan McMahan, Eider Moore, Daniel Ramage, Seth Hampson, and Blaise~Aguera
  y~Arcas.
\newblock Communication-efficient learning of deep networks from decentralized
  data.
\newblock In \emph{Artificial Intelligence and Statistics}, pages 1273--1282.
  PMLR, 2017.

\bibitem[Nesterov(2000)]{nesterov00}
Yurii Nesterov.
\newblock Squared functional systems and optimization problems.
\newblock In \emph{High performance optimization}, pages 405--440. Springer,
  2000.

\bibitem[Nesterov and Polyak(2006)]{nesterovP06}
Yurii Nesterov and Boris~T Polyak.
\newblock Cubic regularization of newton method and its global performance.
\newblock \emph{Mathematical Programming}, 108\penalty0 (1):\penalty0 177--205,
  2006.

\bibitem[Pearlmutter(1994)]{Pearlmutter94product}
Barak~A. Pearlmutter.
\newblock Fast exact multiplication by the hessian.
\newblock \emph{Neural Comput.}, 6\penalty0 (1):\penalty0 147–160, January
  1994.
\newblock ISSN 0899-7667.
\newblock \doi{10.1162/neco.1994.6.1.147}.
\newblock URL \url{https://doi.org/10.1162/neco.1994.6.1.147}.

\bibitem[Reddi et~al.(2017)Reddi, Zaheer, Sra, Poczos, Bach, Salakhutdinov, and
  Smola]{reddi2017generic}
Sashank~J Reddi, Manzil Zaheer, Suvrit Sra, Barnabas Poczos, Francis Bach,
  Ruslan Salakhutdinov, and Alexander~J Smola.
\newblock A generic approach for escaping saddle points.
\newblock \emph{arXiv preprint arXiv:1709.01434}, 2017.

\bibitem[Schraudolph(2002)]{schraudolph2002product}
Nicol~N Schraudolph.
\newblock Fast curvature matrix-vector products for second-order gradient
  descent.
\newblock \emph{Neural computation}, 14\penalty0 (7):\penalty0 1723--1738,
  2002.

\bibitem[Stich(2018)]{stich2018local}
Sebastian~U Stich.
\newblock Local sgd converges fast and communicates little.
\newblock \emph{arXiv preprint arXiv:1805.09767}, 2018.

\bibitem[Stich et~al.(2018)Stich, Cordonnier, and Jaggi]{stich2018sparsified}
Sebastian~U Stich, Jean-Baptiste Cordonnier, and Martin Jaggi.
\newblock Sparsified sgd with memory.
\newblock In \emph{Advances in Neural Information Processing Systems}, pages
  4447--4458, 2018.

\bibitem[Strom(2015)]{strom2015scalable}
Nikko Strom.
\newblock Scalable distributed dnn training using commodity gpu cloud
  computing.
\newblock In \emph{Sixteenth Annual Conference of the International Speech
  Communication Association}, 2015.

\bibitem[{Sun} et~al.(2016){Sun}, {Qu}, and {Wright}]{sunQW16phase}
J.~{Sun}, Q.~{Qu}, and J.~{Wright}.
\newblock A geometric analysis of phase retrieval.
\newblock In \emph{2016 IEEE International Symposium on Information Theory
  (ISIT)}, pages 2379--2383, 2016.

\bibitem[Tripuraneni et~al.(2018)Tripuraneni, Stern, Jin, Regier, and
  Jordan]{tripuraneni2018stochastic}
Nilesh Tripuraneni, Mitchell Stern, Chi Jin, Jeffrey Regier, and Michael~I
  Jordan.
\newblock Stochastic cubic regularization for fast nonconvex optimization.
\newblock In \emph{Advances in Neural Information Processing Systems}, pages
  2904--2913, 2018.

\bibitem[Xu et~al.(2018)Xu, Jin, and Yang]{xu2018first}
Yi~Xu, Rong Jin, and Tianbao Yang.
\newblock First-order stochastic algorithms for escaping from saddle points in
  almost linear time.
\newblock In \emph{Advances in Neural Information Processing Systems}, pages
  5530--5540, 2018.

\bibitem[Zhou and Gu(2019)]{zhou2019stochastic}
Dongruo Zhou and Quanquan Gu.
\newblock Stochastic recursive variance-reduced cubic regularization methods.
\newblock \emph{arXiv preprint arXiv:1901.11518}, 2019.

\bibitem[Zhou et~al.(2018)Zhou, Xu, and Gu]{zhou2018finding}
Dongruo Zhou, Pan Xu, and Quanquan Gu.
\newblock Finding local minima via stochastic nested variance reduction.
\newblock \emph{arXiv preprint arXiv:1806.08782}, 2018.

\end{thebibliography}
\newpage
\clearpage
\onecolumn
\appendix
\section{Convergence to \texorpdfstring{$\eps$}{eps}-FOSP}
\label{sec:fosp}

In this section we prove Theorem~\ref{thm:fosp}, showing that Algorithm~\ref{alg:arbitrary_compressor} converges to an approximate first-order stationary point.
Results and proofs are inspired by~\citet{karimireddy2019error}, with the key difference in that we show how to avoid using the bounded gradient assumption: $\Exp{\|\nabla F\|^2} \le G^2$ and handle the case of $\compr$-compressors with $\compr \ll 1$.
Furthermore, Compressed Descent Lemma (Lemma~\ref{lem:descent_lemma}) is a foundation for showing a second-order convergence.

\begin{definition}[Noise and compression parameters]
\label{def:noises}
We use the following notation:
\begin{compactitem}
    \item $\sgdNoise_t = \nabla F(\vx_t, \sgdRand_t) - \nabla f(\vx_t)$ is stochastic gradient noise. This noise has variance $\std^2$
    \item $\ourNoise_t$ is artificial Gaussian noise added at every iteration. This noise has variance $r^2$
    \item $\fullNoise_t = \sgdNoise_t + \ourNoise_t$ is the total noise. This noise has variance $\fullStd^2 = \std^2 + r^2$.
    \item We assume that gradients are compressed using a $\compr$-compressor $C$.
\end{compactitem}
\end{definition}

For the sake of the analysis, similarly to~\citet{karimireddy2019error}, we introduce an auxiliary sequence of corrected iterates $\{\tvx_t\}$, which remove the impact of the compression error.
\begin{definition}[Corrected iterates]
\label{def:corrected_iterates}
The sequence of \emph{corrected iterates} $\{\tvx_t\}$ is defined as
$$\tvx_t =\vx_t - \step \err_t$$ 
\end{definition}

\begin{proposition}
\label{prop:y_diff}
For the sequence $\{\tvx_t\}$, we have $\tvx_{t+1} - \tvx_t = - \step (\nabla f(\vx_t) + \fullNoise_t)$
\end{proposition}

\begin{proof}
Recall that $\err_{t+1} = \nabla f(\vx_t) + \fullNoise_t + \err_t - \g_t$ and $\g_t = \C(\nabla f(\vx_t) + \fullNoise_t + \err_t, \sgdRand_t)$ and thus
$$ \C(\nabla f(\vx_t) + \fullNoise_t + \err_t, \comprRand_t) = \nabla f(\vx_t) + \fullNoise_t + \err_t - \err_{t + 1}.$$

Substituting this into equation for $\tvx_{t+1}$:
\begin{align*}
    \tvx_{t+1}
    &= \vx_{t+1} - \step \err_{t+1} \\
    &= \vx_{t} - \step \C(\nabla f(\vx_t) + \fullNoise_t + \err_t, \comprRand_t) - \step \err_{t+1}
        && \text{(Since $\vx_{t + 1} = \vx_t - \step \C(\nabla f(\vx_t) + \fullNoise_t + \err_t, \comprRand_t) $)}\\
    &= \vx_t - \step (\nabla f(\vx_t) + \fullNoise_t + \err_t - \err_{t+1}) - \step \err_{t+1}
        \\
    &= \vx_t - \step (\nabla f(\vx_t) + \fullNoise_t + \err_t)  \\
    &= \vx_t - \step \err_t - \step (\nabla f(\vx_t) + \fullNoise_t)  \\
    &= \tvx_t - \step (\nabla f(\vx_t) + \fullNoise_t)
\end{align*}
\end{proof}

\subsection{Compression Error Bound}

Recall that the compression error terms $\err_t$ in Algorithm~\ref{alg:arbitrary_compressor} represent the difference between the computed gradient and the compressed gradient.
Similarly to how stochastic noise increases the number of iterations compared with deterministic gradient descent, compression errors also increase the number of iterations, and therefore it's important to bound $\|\err_t\|$.

\begin{lemma}[Compression Error Bound]
\label{lem:error_estimation}
Let $\vx_t, \err_t$ be defined as in Algorithm~\ref{alg:arbitrary_compressor} and let $\fullStd^2$ be as in Definition~\ref{def:noises}. Then under Assumptions~\ref{ass:lip} and~\ref{ass:sg}, for any $t$ we have
\begin{align*}
\Exp{\|\err_{t}\|^2} \le \frac {2 (1 - \compr)} {\compr} \sum_{i=0}^{t-1} \left(1 - \frac \compr 2\right)^{t-i} \Exp{\|\nabla f(\vx_i)\|^2 + \fullStd^2},
\end{align*}
\end{lemma}

In particular, by considering a uniform bound on $\Exp{\|\nabla f(\vx_i)\|^2}$ and taking the sum of the geometric series, we get a result similar to \citet[Lemma 3]{karimireddy2019error}:
\begin{align*}
\Exp{\|\err_{t}\|^2} \le \frac {4 (1 - \compr)} {\compr^2} \left(\max_{i=0}^{t-1} \Exp{\|\nabla f(\vx_i)\|^2} + \fullStd^2\right)
\end{align*}
\begin{proof}
The proof is similar to the one of~\citet[Lemma~3]{karimireddy2019error}.
The main difference is that we don't rely on the bounded gradient assumption.

By definition of $\err_{t+1}$:
\begin{align*}
\Exp{\|\err_{t+1}\|^2}
&=\Exp{\|\err_{t} + \nabla f(\vx_t) + \fullNoise_t - \C(\err_t + \nabla f(\vx_t) + \fullNoise_t, \sgdRand_t)\|^2} \\
&\le (1 - \compr) \Exp{\|\err_{t} + \nabla f(\vx_t) + \fullNoise_t\|^2} \\
\end{align*}
By using inequality $\|a + b\|^2 \le (1 + \nu) \|a\|^2 + (1 + \frac 1\nu) \|b\|^2$ for any $\nu$:

\begin{align*}
\Exp{\|\err_{t+1}\|^2}
&\le (1 - \compr) ((1+\nu)\Exp{\|\err_{t}\|^2} + (1 +\frac 1\nu)\Exp{\|\nabla f(\vx_t) + \fullNoise_t\|^2}) \\
&\le \sum_{i=0}^t  (1 - \compr)^{t-i+1} (1+\nu)^{t-i} (1 +\frac 1\nu)\Exp{\|\nabla f(\vx_i) + \fullNoise_i\|^2} &&\text{(Telescoping)} \\
&\le \frac 1\nu \sum_{i=0}^t  \left((1 - \compr) (1+\nu)\right)^{t-i+1} \Exp{\|\nabla f(\vx_i) + \fullNoise_i\|^2} \\
\end{align*}

By selecting $\nu= \frac {\compr} {2 (1 - \compr)}$, we have $(1 - \compr) (1 + \nu) = 1 - \frac \compr 2$.
Therefore:
\begin{align*}
\Exp{\|\err_{t+1}\|^2}
&\le \frac {2 (1 - \compr)} {\compr} \sum_{i=0}^t  \left(1 - \frac \compr 2 \right)^{t-i+1} \Exp{\|\nabla f(\vx_i) + \fullNoise_i\|^2}\\
&= \frac {2 (1 - \compr)} {\compr} \sum_{i=0}^t  \left(1 - \frac \compr 2 \right)^{t-i+1} \Exp{\|\nabla f(\vx_i)\|^2 + \fullStd^2} && (\Exp{\fullStd \mid \vx_i} = 0)
\end{align*}
\end{proof}

For the sum of $\|\err_t\|^2$, we have the following, simpler expression:
\begin{corollary}
\label{cor:error_sum}
Under assumptions of Lemma~\ref{lem:error_estimation}, we have
\begin{align*}
\sum_{\tau=0}^t \Exp{\|\err_{\tau}\|^2} \le \frac {4 (1 - \compr)} {\compr^2} \sum_{\tau=0}^t (\Exp{\|\nabla f(\vx_\tau)\|^2} + \fullStd^2)
\end{align*}
\end{corollary}
\begin{proof}
\begin{align*}
    \sum_{\tau=0}^t \Exp{\|\err_{\tau}\|^2}
    &\le \frac{2(1 - \compr)}{\compr} \sum_{\tau=0}^t \sum_{i=0}^{\tau}  \left(1 - \frac \compr 2\right)^{\tau-i+1} \Exp{\|\nabla f(\vx_i)\|^2 + \fullStd^2} \\
    &\le \frac{2(1 - \compr)}{\compr} \sum_{i=0}^t \left( \Exp{\|\nabla f(\vx_i)\|^2 + \fullStd^2} \sum_{\tau=i}^{t}  \left(1 - \frac \compr 2\right)^{\tau-i+1} \right)
\end{align*}    
Bounding $\sum_\tau \left(1 - \frac \compr 2\right)^{\tau-i+1}$ with the sum of the geometric series $\frac 2 \compr$, we have:
\begin{align*}
    \sum_{\tau=0}^t \Exp{\|\err_{\tau}\|^2}
    \le \frac {4 (1 - \compr)} {\compr^2} \sum_{i=0}^t \Exp{\|\nabla f(\vx_i)\|^2 + \fullStd^2}
\end{align*}    
\end{proof}

\subsection{Compressed Descent Lemma}

The following descent lemma is the key tool in the analysis as it allows us to bound gradient norms across multiple iterations.
\begin{lemma}[Compressed Descent Lemma]
\label{lem:descent_lemma}
Let $f$ satisfy Assumptions~\ref{ass:lip} and~\ref{ass:sg} and $\fullStd^2$ be as in Definition~\ref{def:noises}. For $\step < \frac 1 {4 \gradLip} \min(\frac \compr \sqrtomc, 1)$, for any $T$ we have:
\begin{align*}
\sum_{\tau=0}^{T-1} \Exp{\|\nabla f(\vx_\tau)\|^2}
&\le \frac {4 (f(\tvx_0) - \Exp{f(\tvx_T)})} {\step}
     + \step T \fullStd^2 \left(2 \gradLip + \frac {8 \gradLip^2 \step (1 - \compr)} {\compr^2} \right)
\end{align*}
\end{lemma}
Using this lemma, we'll later show that for sufficiently large $T$, multiple visited points have small gradients (note that by dividing the left-hand side by $T$ we obtain an average squared gradient norm), making them $\eps$-FOSP.
On the right-hand side the first term is bounded by $\nicefrac {4 \fmax} \step$, while the other two terms can be bounded by selecting a sufficiently small $\step$. The second term arises from stochastic gradient noise, while the last term stems from the compression error.

\begin{proof}
The proof is similar to the one of~\citet[Theorem~II]{karimireddy2019error}.
By the folklore descent lemma:
\begin{align*}
&\ExpArg{\sgdRand_t,\comprRand_t}{f(\tvx_{t+1}) \mid \vx_t, \err_t} \\
&\le f(\tvx_t) + \langle \nabla f(\tvx_t), \ExpArg{\sgdRand_t,\comprRand_t}{\tvx_{t+1} - \tvx_t \mid \vx_t, \err_t} \rangle
    + \frac L2 \ExpArg{\sgdRand_t,\comprRand_t}{\|\tvx_{t+1} - \tvx_t\|^2 \mid \vx_t, \err_t} \\
&= f(\tvx_t) - \step \ExpArg{\sgdRand_t,\comprRand_t}{\langle \nabla f(\tvx_t), \nabla f(\vx_t) + \fullNoise_t \rangle \mid \vx_t, \err_t}
    + \frac {\gradLip \step^2}2 \ExpArg{\sgdRand_t,\comprRand_t}{\|\nabla f(\vx_t) + \fullNoise_t\|^2 \mid \vx_t, \err_t}
    && (Prop.~\ref{prop:y_diff})\\
&\le f(\tvx_t) - \step \|\nabla f(\vx_t)\|^2
    - \step \langle \nabla f(\tvx_t) - \nabla f(\vx_t), \nabla f(\vx_t) \rangle
    + \frac {\gradLip \step^2}2 \|\nabla f(\vx_t)\|^2 + \frac {\gradLip \step^2 \fullStd^2} 2
    && (\Exp{\fullNoise_t} = 0)\\
&\le f(\tvx_t) - \step \left(1 - \frac {\gradLip \step}2\right) \|\nabla f(\vx_t)\|^2 + \frac {\gradLip \step^2 \fullStd^2}2 - \step \langle \nabla f(\tvx_t) - \nabla f(\vx_t), \nabla f(\vx_t) \rangle \\
\end{align*}

Using inequality $|\langle a, b \rangle| \le \frac {\|a\|^2} 2 + \frac {\|b\|^2} 2$ and smoothness, we have:
\begin{align*}
&\ExpArg{\sgdRand_t,\comprRand_t}{f(\tvx_{t+1}) \mid \vx_t, \err_t}\\
&\le f(\tvx_t) - \step \left(1 - \frac {\gradLip \step}2 \right) \|\nabla f(\vx_t)\|^2
    + \frac {\gradLip \step^2 \fullStd^2}2 + \frac \step 2 \|\nabla f(\tvx_t) - \nabla f(\vx_t)\|^2 + \frac \step 2 \|\nabla f(\vx_t)\|^2 \\
&\le f(\tvx_t) - \step \left(\frac 12 - \frac {\gradLip \step}2 \right) \|\nabla f(\vx_t)\|^2
    + \frac {\gradLip \step^2 \fullStd^2}2 + \frac {\step \gradLip^2} 2 \|\tvx_t -\vx_t\|^2
    && (\text{$\gradLip$-smoothness})\\
&\le f(\tvx_t) - \step \left(\frac 12 - \frac {\gradLip \step}2 \right) \|\nabla f(\vx_t)\|^2
    + \frac {\gradLip \step^2 \fullStd^2}2 + \frac {\step^3 \gradLip^2} 2 \|\err_t\|^2
    && (\text{Def.~\ref{def:corrected_iterates} of $\tvx_t$})
\end{align*}

Using telescoping and taking the expectation, we bound $f(\tvx_{t+1})$:
\begin{align*}
\Exp{f(\tvx_{t+1})}
\le f(\tvx_0) - \step \left(\frac 12 - \frac {\gradLip \step}2 \right) \sum_{\tau=0}^t \Exp{\|\nabla f(\vx_\tau)\|^2}
    + \frac {\gradLip \fullStd^2 \step^2 (t+1)}2
    + \step^3 \gradLip^2 \sum_{\tau=0}^t \Exp{\|\err_\tau\|^2}
\end{align*}


Bounding $\sum_\tau \|\err_\tau\|^2$ by Corollary~\ref{cor:error_sum}, we have:
\begin{align*}
&\Exp{f(\tvx_{t})}\\
&\le f(\tvx_0) - \step \left(\frac 12 - \frac {\gradLip \step}2 \right) \sum_{\tau=0}^{t-1} \Exp{\|\nabla f(\vx_\tau)\|^2}
        + \frac {\gradLip \fullStd^2 \step^2 t}2
        + \frac {2 \step^3 \gradLip^2 (1 - \compr)} {\compr^2} \sum_{i=0}^{t-1} \Exp{\|\nabla f(\vx_i)\|^2 + \fullStd^2} \\
&\le f(\tvx_0) - \step \left(\frac 12 - \frac {\gradLip \step}2 - \frac {2 \step^2 \gradLip^2 (1 - \compr)} {\compr^2} \right) \sum_{\tau=0}^{t-1} \Exp{\|\nabla f(\vx_\tau)\|^2}
        + \frac {\gradLip \fullStd^2 \step^2 t}2
        + \frac {2 \step^3 \gradLip^2 \fullStd^2 (1 - \compr) t} {\compr^2}
\end{align*}

Using that $\step < \frac 1 {4 \gradLip} \min\left(\frac \compr \sqrtomc, 1\right)$,
we bound the coefficient before $\sum_{\tau=0}^t \Exp{\|\nabla f(\vx_\tau)\|^2}$ with $\frac \step 4$:
\begin{align*}
\Exp{f(\tvx_t)}
&\le f(\tvx_0) - \frac \step 4 \sum_{\tau=0}^{t-1} \Exp{\|\nabla f(\vx_\tau)\|^2}
        + \step^2 \fullStd^2 t \left(\frac \gradLip 2
        + \frac {2 \gradLip^2 \step (1 - \compr)} {\compr^2}\right)
\end{align*}
After regrouping the terms, we get the final result:
\begin{align*}
\sum_{\tau=0}^{t-1} \Exp{\|\nabla f(\vx_\tau)\|^2}
&\le \frac {4 (f(\tvx_0) - \Exp{f(\tvx_t)})} {\step}
        + \step \fullStd^2 t \left(2\gradLip + \frac {8 \gradLip^2 \step (1 - \compr)} {\compr^2}\right)
\end{align*}
\end{proof}

\subsection{Convergence to \texorpdfstring{$\eps$}{eps}-FOSP}

\begin{theorem}[Convergence to $\eps$-FOSP]
Let $f$ satisfy Assumptions~\ref{ass:lip} and~\ref{ass:sg}.
Then for $\step=\tilde O\left(\min \left(\eps^2, \frac \compr \sqrtomc \eps \right)\right)$,
after $T=\tilde \Theta\left(\frac 1 {\eps^4} + \frac {\sqrtomc} {\compr \eps^3}\right)$ iterations, at least half of visited points are $\eps$-FOSP\@.
\end{theorem}

\begin{proof}
Proof by contradiction.
For $\step < \frac 1 {4 \gradLip} \min \left(\frac \compr \sqrtomc, 1\right)$, if less than half points are $\eps$-FOSP, then by Lemma~\ref{lem:descent_lemma}:
\begin{align*}
\frac {T \eps^2} 2
\le \sum_{\tau=0}^T \Exp{\|\nabla f(\vx_\tau)\|^2}
&\le \frac {4 \fmax} {\step} + \step \fullStd^2 T \left(2\gradLip + \frac {8 \gradLip^2 \step (1 - \compr)} {\compr^2}\right)
\end{align*}

It suffices to guarantee that all terms on the right-hand side are at most $\frac {T \eps^2} 6$:
\begin{align*}
2 \gradLip \step \fullStd^2 T \le \frac {T \eps^2} 6
    &\iff \step \le \frac {\eps^2} {12 \gradLip \fullStd^2}
    &&=\tilde \Theta(\eps^2) \\
\frac {8 \gradLip^2 \fullStd^2 \step^2 T (1 - \compr)} {\compr^2} \le \frac {T \eps^2} 6
    &\iff \step \le \frac {\compr \eps} {\sqrtomc \gradLip \fullStd \sqrt{48}}
    &&=\tilde \Theta\left(\frac {\compr \eps}{\sqrtomc}\right) \\
\frac {4 \fmax}{\step} \le \frac {T \eps^2} 6
    &\iff T \ge \frac {24 \fmax} {\eps^2 \step}
    &&= \tilde \Theta\left(\frac 1 {\step \eps^2}\right) = \tilde \Theta\left(\frac 1 {\eps^4} + \frac {\sqrtomc} {\compr \eps^3}\right)
\end{align*}

Therefore, after $\tilde \Theta\left(\frac 1 {\eps^4} + \frac \sqrtomc {\compr \eps^3}\right)$ iterations at least half of the points are $\eps$-FOSP\@.
\end{proof}

\newpage
\section{Convergence to \texorpdfstring{$\eps$}{eps}-SOSP}
\label{sec:sosp}

By rescaling we can assume that $\eps \le 1$.
Recall that $\alpha = 1$ when Assumption~\ref{ass:lipsg} holds and $\alpha = d$ otherwise.
We introduce the following auxiliary notation:
\begin{definition}[Step sizes]
\begin{align*}
    \text{$\max$ $\step$ for SGD} && \step_\std
    &= \frac {\eps^2} {\gradLip (1 + d \std^2)} + \min\left(\frac {\eps^2} {\gradLip(1 + \std^2)}, \frac {\sqrtre} {\sgdLip^2}\right)
    = \tilde O \left(\frac{\eps^2}{\alpha}\right) \\
    \text{$\max$ $\step$ for compressed SGD:} && \\
    \text{For a general compressor:} &&\step_\compr
    &= \min \left(\frac {\compr \eps} {\sqrtomc \gradLip \std},
        \frac  {\compr^2 \sqrt{\eps}} {(1 - \compr) \gradLip^2 d} \right)
    = \tilde O \left(\min \left(\frac {\compr \eps} {\sqrtomc}, \frac {\compr^2 \sqrt{\eps}} {(1 - \compr) d} \right)\right) \\
    \text{For a linear compressor:} &&\step_\compr
    &= \min \left(\frac {\compr \eps} {\sqrtomc \gradLip \std},
    \frac  {\compr^2 \sqrtre} {(1 - \compr) \gradLip^2} \right)
    = \tilde O \left(\min \left(\frac {\compr \eps} {\sqrtomc}, \frac {\compr^2 \sqrt{\eps}} {1 - \compr} \right)\right) \\
\end{align*}
\end{definition}


Intuitively, selecting step size $\step \le \step_\std$ suffices to show convergence of SGD~\citep{jin2019nonconvex}.
In addition, selecting $\step \le \step_\compr$ allows us to extend the results to compressed SGD\@.
\begin{definition}
Our choice of parameters is the following ($\cstep, \cit, \crad, \cf, \cnoise$ hide polylogarithmic dependence on all parameters):
\begin{equation}
    \begin{aligned}
        \text{Step size} && \step &= \cstep \min(\step_\std, \step_\compr) \\
        \text{Iterations required for escaping} && \escIter &=  \cit \frac {1} {\step \sqrtre} \\
        \text{Escaping radius} && \escRad &= \crad \sqrt {\frac {\eps} \hesLip} \\
        \text{Objective change after escaping} && \df &= \cf \sqrt {\frac {\eps^3} \hesLip} \\
        \text{Noise standard deviation} && r &= \cnoise \frac \eps {\sqrt{\gradLip \step}}
    \end{aligned}\label{eq:params}
\end{equation}
\end{definition}

Recall that $\fullStd^2 = \std^2 + r^2 = \std^2 + \frac {\cnoise  \eps^2} {\gradLip \step}$ by Definition~\ref{def:noises} and $\fmax = f(\vx_0) - f(\vx^*)$.
We will show that after $\escIter$ iterations the objective decreases by $\df$.
Therefore, the objective decreases on average by $\frac{\df}{\escIter} = \tilde \Omega(\eps^2 \step)$ per iteration
resulting in $\tilde O\left(\frac {\fmax} {\eps^2 \step}\right)$ iterations overall. See Table~\ref{tab:results} for the number of iterations and total communication in various settings. 

Intuitively, the motivation for this choice of parameters is the following. Let $\vx$ be a point such that $\lmin (\nabla^2 f(\vx)) < -\sqrtre$ and $\|\nabla f(\vx)\|=0$.
\begin{itemize}
    \item
        Our analysis happens inside $B(\vx, \escRad)$, and inside this ball we want $\lmin (\nabla^2 f) < -\frac \sqrtre 2$.
        By the Hessian-Lipschitz property, for $\vy \in B(\vx, \escRad)$ we have $\|\nabla^2 f(\vx) - \nabla^2 f(\vy)\| \le \hesLip \escRad$.
        To have $\hesLip \escRad \le \frac \sqrtre 2$, we choose $\escRad \le \frac 12 \sqrt{\frac{\eps}{\hesLip}}$.
    \item
        Let $-\eigen, \vv_1$ be the smallest negative eigenpair of $\nabla f^2(\vx)$
        Assume that our function is quadratic and, after adding noise, the projection on $\vv_1$ is $\Theta(1)$ (it is actually polynomial or reverse-polynomial on all parameters, which doesn't change the idea).
        Then after $t$ iterations, this projection increases by the factor of $(1 + \step \eigen)^t$. For every $\nicefrac 1{\step \eigen}$ iterations, the projection increases approximately by the factor of $e$. Therefore, to reach $\escRad$ starting from $\Theta(1)$, we need $\tilde O(\frac 1 {\step \eigen})$ iterations, which is at most $\tilde O(\frac 1 {\step \sqrtre})$
    \item
        In some sense, the best improvement we can hope to achieve is by moving from $\vx$ to $\vx + \escRad \vv_1$. Since $\nabla f(\vx)=0$ and the objective is quadratic in direction $\vv_1$ with eigenvalue $\eigen$, the objective decreases by $\eigen \escRad^2 = \Theta(\sqrt {\frac {\eps^3} \hesLip})$, which motivates the choice of $\df$.
    \item
        Bound in $r$ arises from the fact that $\fullStd^2 \approx r^2$ and that we want to bound the last term in Lemma~\ref{lem:improve_or_localize} with $\df$.
\end{itemize}

\begin{table}[t!]
    \centering
    \caption{Convergence to $\eps$-SOSP for full SGD and for constant-size compression (the choice of parameters is not optimal; see Table~\ref{tab:results} for the optimal choice).
    For any choice of $\compr$ and $\step$: $T = \tilde O(\nicefrac 1 {\step \eps^2})$, $\escRad = \tilde O(\sqrt{\eps})$, $\df = \tilde O(\sqrt{\eps^3})$.
    }
    \label{tab:params_cases}
    \begin{tabular}{ccccc}
        \hline
        Settings & $\compr$ & $\step$ & $\escIter$ & $r$ \\
        \hline
        \multiline{Uncompressed\\ Lipschitz $\nabla F$}
        & $0$
        & $\tilde O\left(\eps^2\right)$
        & $\tilde O\left(\eps^{\nicefrac 32}\right)$
        & $\tilde O\left(1\right)$ \\
        \hline
        \multiline{Compressed\\ Lipschitz $\nabla F$}
        & $\frac 1d$
        & $\tilde O\left(\min\left(\eps^2, \frac {\eps} {d}, \frac {\sqrt \eps} {d^3}\right)\right)$
        & $\tilde O\left(\frac 1 {\step \sqrt{\eps}}\right)$
        & $\tilde O\left(\frac \eps {\sqrt{\step}} \right)$ \\
        \hline
        \multiline{\textsc{RandomK}\\ Lipschitz $\nabla F$}
        & $\frac 1d$
        & $\tilde O\left(\min\left(\eps^2, \frac {\eps} {d}, \frac {\sqrt \eps} {d^2}\right)\right)$
        & $\tilde O\left(\frac 1 {\step \sqrt{\eps}}\right)$
        & $\tilde O\left(\frac \eps {\sqrt{\step}} \right)$ \\
        \hline
        \hline
        \multiline{Uncompressed\\ non-Lipschitz $\nabla F$}
        & $0$
        & $\tilde O\left(\frac {\eps^2} d\right)$
        & $\tilde O\left(d \eps^{\nicefrac 32}\right)$
        & $\tilde O\left(\sqrt{d}\right)$ \\
        \hline
        \multiline{\textsc{Compressed}\\ non-Lipschitz $\nabla F$}
        & $\frac 1d$
        & $\tilde O\left(\min\left(\frac {\eps^2} d, \frac {\sqrt \eps} {d^3}\right)\right)$
        & $\tilde O\left(\frac 1 {\step \sqrt{\eps}}\right)$
        & $\tilde O\left(\frac \eps {\sqrt{\step}} \right)$ \\
        \hline
        \multiline{\textsc{RandomK}\\ non-Lipschitz $\nabla F$}
        & $\frac 1d$
        & $\tilde O\left(\min\left(\frac {\eps^2} d, \frac {\sqrt \eps} {d^2}\right)\right)$
        & $\tilde O\left(\frac 1 {\step \sqrt{\eps}}\right)$
        & $\tilde O\left(\frac \eps {\sqrt{\step}} \right)$ \\
        \hline
    \end{tabular}
\end{table}

For a linear compressor, we perform analysis for arbitrary points arising in Algorithm~\ref{alg:arbitrary_compressor},
while for an arbitrary compressor, we perform analysis for points with $\err_t = 0$, i.e. from points $\vx_{t'}$ from Algorithm~\ref{alg:arbitrary_compressor}.
To simplify the representation, we assume that $t = 0$, so that we are able to use results from the previous section: this is a valid choice for an arbitrary compressor since $\err_t = 0$, and for a linear compressor our analysis doesn't use $\err_t$.

\subsection{Proof outline}

Our proof is mainly based on the ideas from~\citet{jin2019nonconvex}.
We first introduce "Improve or localize" lemma (Lemma~\ref{lem:improve_or_localize}): if after the limited number of iterations the objective doesn't sufficiently improve, we conclude that we didn't move far from the original point.
Similarly to~\citet{jin2019nonconvex}, we introduce a notion of coupling sequences: two gradient descent sequences having the same distribution such that, as long as we start from a saddle point, at least one of these sequences escapes, and therefore its objective improves.
Since distributions of these sequences match distribution of sequence generated by gradient descent, we conclude that the algorithm sufficiently improves the objective.

Our analysis differs from~\citet{jin2019nonconvex} in several ways.
The first difference is that, aside from $\{\vx_t\}$, our equations use another sequence $\{\tvx_t\}$ ($\vx_t$ mainly participate as arguments of $\nabla f(\cdot)$, while $\tvx_t$ participate as argument of $f(\cdot)$ and in distances).
This leads to the following challenge: if some relation holds for $\tvx_t$, it doesn't necessary holds for $\vx_t$.
For example, if we have a bound on $\|\tvx_t - \tvx'_t\|$, we don't necessarily have a bound on $\|\vx_t - \vx'_t\|$, and it needs to be established separately.

Another difference is that, for a general compressor, we have to split our analysis into two parts: large gradient case and small gradient case.
When our initial gradient is large, then we either escape the saddle points or the nearby gradients are also large, and by Lemma~\ref{lem:descent_lemma} the objective improves (see details in Lemma~\ref{lem:large_gradient}).
If the gradient is small, we use "Improve or localize" Lemma as described above.
In the latter case, similarly to~\citet{jin2019nonconvex}, we have to bound errors which arise from the fact that the function is not quadratic and gradients are not deterministic (see Definition~\ref{def:error-series}).
However, we have an additional error term stemming from gradient compression (see Definition~\ref{def:error-series}); to bound this term (see Lemma~\ref{lem:bound_error_sum}), we need bounded $\|\err_t\|$, and for that we use our assumptions that gradients are small.

\subsection{Improve or localize}

We first show that, if gradient descent moves far enough from the initial point, then function value sufficiently decreases.
The following lemma considers the general case, while Corollary~\ref{cor:improve_or_localize _our_params} considers the simplified form,
obtained by substituting parameters from Equation~\ref{eq:params}.

\begin{lemma}[Improve or localize]
    \label{lem:improve_or_localize}
    Under Assumptions~\ref{ass:lip} and~\ref{ass:sg}, for $\step < \frac 1 {4 \gradLip} \min(\frac \compr \sqrtomc, 1)$, for $\tvx_t$ and $\fullStd$ defined as in Definition~\ref{def:noises}, we have
    \begin{align*}
        f(\tvx_0) - \Exp{f(\tvx_t)} \ge \frac {\Exp{\|\tvx_t - \tvx_0\|^2}} {8 \step t}
        - \step^2 \fullStd^2 t \left(\gradLip + \frac {2 (1 - \compr) \gradLip^2 \step} {\compr^2} \right)
        - \step \fullStd^2
    \end{align*}
\end{lemma}
\begin{proof}
    Let $\fullNoise_t = \sgdNoise_t + \ourNoise_t$.
    By Proposition~\ref{prop:y_diff}, $\tvx_{i+1} = \tvx_i - \step (\nabla f(\vx_i) + \fullNoise_i)$.
    Since noises are independent:
    \[\Exp{\|\sum_{i=0}^{t-1} \fullNoise_i\|^2}
    = \sum_{i=0}^{t-1} \Exp{\|\fullNoise_i\|^2}
    = \sum_{\tau=0}^{t-1} \fullStd^2
    = t \fullStd^2\]

    \begin{align*}
        \Exp{\|\tvx_t - \tvx_0\|^2}
        &= \step^2 \Exp{\|\sum_{i=0}^{t-1} (\nabla f(\vx_i) + \fullNoise_i)\|^2} \\
        &\le 2\step^2 \Exp{\|\sum_{i=0}^{t-1} \nabla f(\vx_i)\|^2 + \|\sum_{i=0}^{t-1} \fullNoise_i\|^2} \\
        &\le 2\step^2 t \sum_{i=0}^{t-1} \Exp{\|\nabla  f(\vx_i)\|^2} + 2 \step^2 \fullStd^2 t \\
    \end{align*}

    Since $\step < \frac 1 {4 \gradLip} \min(\frac \compr \sqrtomc, 1)$, by Lemma~\ref{lem:descent_lemma}:
    \begin{align*}
        \Exp{\|\tvx_t - \tvx_0\|^2}
        &\le 2\step^2 t \left(\frac {4 (f(\tvx_0) - \Exp{f(\tvx_t)})} \step
        + \step \fullStd^2 t \left(2\gradLip + \frac {8 (1 - \compr) \gradLip^2 \step} {\compr^2}\right) \right)
        + 2 \step^2 \fullStd^2 t \\
        &\le 2 \step t \left(4(f(\tvx_0) - \Exp{f(\tvx_t)})
        + \step^2 \fullStd^2 t \left(4\gradLip + \frac {8 (1 - \compr) \gradLip^2 \step} {\compr^2}\right) + 4 \step \fullStd^2\right) \\
    \end{align*}
    After regrouping the terms, we have:
    \begin{align*}
        f(\tvx_0) - \Exp{f(\tvx_t)}
        \ge \frac {\Exp{\|\tvx_t - \tvx_0\|^2}} {8 \step t}
        - \step^2 \fullStd^2 t \left(\gradLip + \frac {2 (1 - \compr) \gradLip^2 \step} {\compr^2}\right)
        - \step \fullStd^2,
    \end{align*}


\end{proof}

\begin{corollary}
    \label{cor:improve_or_localize _our_params}
    Under Assumptions~\ref{ass:lip} and~\ref{ass:sg}, for $\df, \escIter$ chosen as specified in Equation~\ref{eq:params}, for any $t \le \escIter$ we have:
    \begin{align*}
        f(\tvx_0) - \Exp{f(\tvx_t)} \ge \frac{\sqrtre}{8 \cit} \Exp{\|\tvx_t - \tvx_0\|^2} - \df
    \end{align*}
\end{corollary}
\begin{proof}
    The first term follows from $t \le \escIter = \frac \cit {\step \sqrtre}$.
    With our choice of parameters, we can bound negative terms on the right-hand side of Lemma~\ref{lem:improve_or_localize} with $\df$ (recall that $\df = \cf \sqrt {\frac {\eps^3} \hesLip}$).

    \paragraph{Bounding $\step \fullStd^2$.}
    \begin{align*}
        \step \fullStd^2
        = \step \std^2 + \step r^2
        \le \cstep \frac{\eps^2}{\gradLip} + \cnoise^2 \frac {\eps^2} {\gradLip}
        = (\cstep + \cnoise^2) \frac {\sqrt{\eps^3}} {\sqrt \hesLip} \cdot \frac{\sqrtre} {\gradLip}
        \le (\cstep + \cnoise^2) \frac {\sqrt{\eps^3}} {\sqrt \hesLip},
    \end{align*}
    where we use that $\sqrtre \le \gradLip$, since otherwise all $\eps$-FOSP are $\eps$-SOSP.

    \paragraph{Bounding $\step^2 \fullStd^2 t \gradLip$.} Since $\escIter =  \cit \frac {1} {\step \sqrtre}$ and $t \le \escIter$:
    \begin{align*}
        \step^2 \fullStd^2 t \gradLip
        \le \frac{\step \fullStd^2 \gradLip}{\sqrtre}
        \le \step \fullStd^2,
    \end{align*}
    and we use the estimation above.

    \paragraph{Bounding $\step^2 \fullStd^2 t \cdot \frac {2 (1 - \compr) \gradLip^2 \step} {\compr^2}$.}
    \begin{align*}
        \frac {\step^3 \fullStd^2 t (1 - \compr) \gradLip^2} {\compr^2}
        &\le \frac {\cit \step^2 \fullStd^2 (1 - \compr) \gradLip^2} {\compr^2 \sqrtre}
            && (t \le \escIter = \frac \cit {\step \sqrtre}) \\
        &\le \frac {\cit \step^2 (1 - \compr) \gradLip^2 \left(\std^2 + \frac {\cnoise  \eps^2} {\gradLip \step}\right)} {\compr^2 \sqrtre}
            && (\fullStd^2 = \std^2 + \frac {\cnoise \eps^2} {\gradLip \step})\\
        &\le \frac {\cit (1 - \compr)} {\compr^2 \sqrtre} \left(\step_\compr^2 \gradLip^2 \std^2 + \cnoise \step_\compr \gradLip \eps^2 \right)
            && (\step \le \step_\compr)\\
        &\le 2 \cit \cstep \frac{\sqrt{\eps^3}}{\sqrt \hesLip}
            && (\step_\compr \le \frac {\compr \eps} {\sqrtomc \gradLip \std} \text{ and } \step_\compr \le \frac {\compr^2 \sqrt{\eps}} {1 - \compr})
    \end{align*}
    
    To guarantee that the sum of these terms is at most $\df$, it suffices to select parameters so that $\cstep + \cnoise^2 + \cit \cstep \le \nicefrac{\cf}{2}$.
\end{proof}

\begin{corollary}
    \label{cor:escaping_improves}
    Under Assumptions~\ref{ass:lip} and~\ref{ass:sg}, for $\df, \escRad, \escIter$ chosen as specified in Equation~\ref{eq:params}, if there exists $t \in [0, \escIter]$ such that $\Exp{\|\tvx_t - \tvx_0\|^2} > \escRad^2$, then $f(\tvx_0) - \Exp{f(\tvx_t)} \ge \df$.
\end{corollary}

\begin{proof}
    By Lemma~\ref{lem:improve_or_localize}, since $\escRad = \crad \sqrt {\frac {\eps} \hesLip}$ and $\df = \cf \sqrt {\frac {\eps^3} \hesLip}$:
    \begin{align*}
        f(\tvx_0) - \Exp{f(\tvx_t)} \ge \frac {\sqrtre \escRad^2} {8 \cit} - \df
        = \frac {\crad^2 \eps \step \sqrtre} {8 \cit \step \hesLip} - \df
        = \left(\frac {\crad^2} {8 \cit \cf} - 1\right) \df
        \ge \df,
    \end{align*}
    where the last inequality holds when $16 \cit \cf \le \crad^2$.
\end{proof}


\subsection{Large gradient case\texorpdfstring{: $\|\nabla f(\vx_0)\| \ge 4 \gradLip \escRad$}{}}

In this section, we consider the case when the gradient is large, and therefore we can make sufficient progress simply by the Compressed Descent Lemma.
Note that the results from this section are only required when the compressor is not linear.


\begin{lemma}[Large gradient case]
    \label{lem:large_gradient}
    Under Assumptions~\ref{ass:lip} and~\ref{ass:sg}, for $\df, \escRad, \escIter$ chosen as specified in Equation~\eqref{eq:params},
    if $\|\nabla f(\vx_0)\| > 4 \gradLip \escRad$, then after at most $\escIter$ iterations the objective decreases by $\df$.
\end{lemma}
\begin{proof}
    If there exists $t \le \escIter$ such that $\Exp{\|\tvx_t - \tvx_0\|^2} > \escRad^2$, then by Corollary~\ref{cor:escaping_improves}, the objective decreases by at least $\df$.
    
    Consider the case when $\Exp{\|\tvx_t - \tvx_0\|^2} \le \escRad^2$ for all $t$. First, to bound the error term, we show by induction that $\Exp{\|\nabla f(\vx_t)\|^2} \le 4 \|\nabla f(\vx_0)\|^2$ for all $t \le \escIter$.
    \begin{align*}
        \|\nabla f(\vx_t)\|^2
        &= \|\nabla f(\vx_0) - (\nabla f(\vx_0) - \nabla f(\vx_t)) \|^2 \\
        &\le 2 \|\nabla f(\vx_0)\|^2 + 2 \|\nabla f(\vx_0) - \nabla f(\vx_t)\|^2
            && (\|a + b\|^2 \le 2 (\|a\|^2 + \|b\|^2) ) \\
        &\le 2 \|\nabla f(\vx_0)\|^2 + 2 \gradLip^2 \|\vx_0 - \vx_t\|^2
            && \text{(Smoothness)} \\
        &\le 2 \|\nabla f(\vx_0)\|^2 + 4 \gradLip^2 \|\tvx_0 - \tvx_t\|^2 + 4 \gradLip^2 \|\tvx_t - \vx_t\|^2
            && \text{(Same inequality and $\vx_0=\tvx_0$)} \\
        &\le 2 \|\nabla f(\vx_0)\|^2 + 4 \gradLip^2 \|\tvx_0 - \tvx_t\|^2 + 4 \gradLip^2 \step^2 \|\err_t\|^2
            && \text{(Definition~\ref{def:corrected_iterates} of $\tvx_t$)}
    \end{align*}
    
    By Lemma~\ref{lem:error_estimation} and the induction hypothesis, we have:
    \[\Exp{\|\err_{t}\|^2}
        \le \frac {4 (1 - \compr)} {\compr^2} \left(\max_{i=0}^{t-1} \Exp{\|\nabla f(\vx_i)\|^2} + \fullStd^2\right)
        \le \frac {4 (1 - \compr)} {\compr^2} \left(4 \|\nabla f(\vx_0)\|^2 + \fullStd^2\right),\]
        
    and therefore for $\step$ chosen as in Equation~\eqref{eq:params}, $\Exp{\gradLip^2 \step^2 \|\err_t\|^2} \le \frac {\|\nabla f(\vx_0)\|^2} 4$.
    By taking an expectation in the equation above, we have:
    \begin{align*}
        \Exp{\|\nabla f(\vx_t)\|^2}
        \le 2 \|\nabla f(\vx_0)\|^2  + 4 \gradLip^2 \escRad^2 + \frac {\|\nabla f(\vx_0)\|^2} 4
        \le 4 \|\nabla f(\vx_0)\|^2
    \end{align*}

    
    
    Given the bound on $\|\err_t\|$,  we can give a lower bound on gradient norm:
    \begin{align*}
        \|\nabla f(\vx_t)\|^2
        &= \|\nabla f(\vx_0) - (\nabla f(\vx_0) - \nabla f(\vx_t)) \|^2 \\
        &\ge \|\nabla f(\vx_0)\|^2 + \|\nabla f(\vx_0) - \nabla f(\vx_t)\|^2 - 2 \|\nabla f(\vx_0)\| \cdot \|\nabla f(\vx_0) - \nabla f(\vx_t)\| \\
        &\ge \|\nabla f(\vx_0)\| (\|\nabla f(\vx_0)\| - 2 \|\nabla f(\tvx_0) - \nabla f(\tvx_t)\| - 2 \|\nabla f(\tvx_t) - \nabla f(\vx_t)\|)
    \end{align*}
    
    By taking expectations and using that $\Exp{\|x\|} \le \sqrt{\Exp{\|x\|^2}}$, and using bound on $\|\err_t\|$, we have:
    \[\Exp{\|\nabla f(\vx_t)\|^2} \ge \nabla f(\vx_0) (\|\nabla f(\vx_0)\| - 2 \gradLip \escRad - \frac {\|\nabla f(\vx_0)\|} 4) \ge 4 \gradLip^2 \escRad^2\]

    
    By Lemma~\ref{lem:descent_lemma}, we know:
    \begin{align*}
        \sum_{\tau=0}^{T-1} \Exp{\|\nabla f(\vx_\tau)\|^2} &\le \frac {4 (f(\tvx_0) - \Exp{f(\tvx_\escIter)})} {\step}
        + \step \fullStd^2 \escIter \left(2\gradLip + \frac {8 (1 - \compr) \gradLip^2 \step} {\compr^2} \right)
    \end{align*}

    Therefore:
    \begin{align*}
        f(\tvx_0) - \Exp{f(\tvx_\escIter)}
        &\ge \frac {\step \escIter} 4 \left(4 \gradLip^2 \escRad^2 - \step \fullStd^2 (2\gradLip + \frac {8 (1 - \compr) \gradLip^2 \step} {\compr^2})\right) \\
        &\ge \step \escIter \gradLip^2 \escRad^2 - \df
            && \text{(See proof of Corollary~\ref{cor:improve_or_localize _our_params})} \\
        &\ge \frac {\cit \step} {\step \sqrtre} \frac {\crad^2 \gradLip^2 \eps} {\hesLip} - \df
            && (\escRad = \crad \sqrt {\frac {\eps} \hesLip} \text{ and } \escIter = \frac \cit {\step \sqrtre})\\
        &\ge \frac {\cit} {\sqrtre} \crad^2 \eps^2 - \df
            && \text{(since $\gradLip \ge \sqrtre$)}\\
        &\ge \df,
    \end{align*}
    where the last inequality holds when $\cit \crad^2 \ge 2 \cf$.
\end{proof}

\subsection{Small Gradient Case\texorpdfstring{: $\|\nabla f(\vx_0)\| < 4 \gradLip \escRad$}{}}

\subsubsection{Coupling Sequences}

Let $H=\nabla^2 f(\vx_0)$, then we can use $\vx^\top H \vx$ as a quadratic approximation of $f$ near $x_0$.
Let $\vv_1$ be the eigenvector corresponding to the smallest eigenvalue $\eigen$ of $H$. Then we construct \emph{coupling sequences} $\vx_t$ and $\vx'_t$ in the following way:
$\vx_t$ is the sequence from Algorithm~\ref{alg:arbitrary_compressor}; $\vx'_t$ has the same stochastic randomness $\sgdRand$ as $\vx_t$, and its artificial noise $\ourNoise'_t$ is the same as $\ourNoise_t$ with exception of the coordinate corresponding to $\vv_1$, which has an opposite sign.
\begin{definition}[Coupling sequences]
The coupling sequences are defined as follows (note the definition of $\ourNoise'_t$):
\begin{equation}
    \begin{aligned}
        & & \err'_0 &= \err_0 \\
        \ourNoise_t &\sim \mathcal{N}(0, r^2)
            & \ourNoise'_t &= \ourNoise_t - 2 \langle \vv_1, \ourNoise_t \rangle \vv_1  \\
        \g_t &= C(\nabla F(\vx_t, \sgdRand_t) + \ourNoise_t + \err_t, \comprRand_t)
            & \g'_t &= C(\nabla F(\vx'_t, \sgdRand_t) + \ourNoise'_t + \err'_t, \comprRand_t), & \sgdRand_t \sim \sgdDistr, \comprRand_t \sim \comprDistr\\
        \tvx_{t} &=\vx_{t} - \step \err_{t}
            & \tvx'_{t} &= \vx'_{t} - \step \err'_{t} \\
        \vx_{t+1} &=\vx_t - \step \g_t
            & \vx'_{t+1} &= \vx'_t - \step \g'_t \\
        \err_{t+1} &= \nabla F(x_t, \sgdRand_t) + \ourNoise_t + \err_t - \g_t
            & \err'_{t+1} &= \nabla F(x'_t, \sgdRand_t) + \ourNoise'_t + \err'_t - g'_t
    \end{aligned}\label{eq:coulping_sequence}
\end{equation}
\end{definition}

A notable fact is that both sequences correspond to the same distribution.

\begin{proposition}
    \label{prop:same_distribution}
    For all $t$, $\vx_t$ and $\tvx_t$ from Equation~\ref{eq:coulping_sequence} have the same distribution as $\vx'_t$ and $\tvx'_t$.
\end{proposition}
\begin{proof}
    By definition of $\tvx_t$ and $\tvx'_t$, it suffices show that $\vx_t$ and $\err_t$ have the same distributions as $\vx'_t$ and $\err'_t$.
    Proof by Induction with trivial base case $\tvx_0 = \tvx'_0 = \vx_0 - \step \err_0$.

    We want to show that if the statement holds for $t$, then it holds for $t+1$.
    To show that $\vx_{t+1}$ has the same distribution it remains to show that $\g_t$ and $\g'_t$ have the same distribution:
    \begin{compactitem}
        \item Since $\vx_t$ and $\vx'_t$ have the same distribution, $\nabla F(\vx_t, \sgdRand_t)$ and $\nabla F(\vx'_t, \sgdRand_t)$ have the same distribution.
        
        \item Since $\mathcal{N}(0, r^2)$ is symmetric and $\ourNoise'_t$ is the same as $\ourNoise_t$ with exception of one coordinate, which has an opposite sign, $\ourNoise_t$ and $\ourNoise'_t$ have the same distribution.
        
        \item $\err_t$ and $\err'_t$ have the same distribution.
    \end{compactitem}
    
    Similarly, $\err_{t+1}$ has the same distribution as $\err'_{t+1}$, since $\nabla F(\vx_t, \sgdRand_t)$, $\ourNoise_t$, $\err_t$ and $\g_t$ have the same distribution as $\nabla F(\vx'_t, \sgdRand_t)$, $\ourNoise'_t$, $\err'_t$ and $\g'_t$.
\end{proof}
Since our sequences have the same distribution, we have $\Exp{f(\vx_t)} = \Exp{f(\vx'_t)}$.
We want to show that in a few iterations $\tvx'_t - \tvx_t$ becomes sufficiently large
and, therefore, at least one of $\tvx_t$ and $\tvx'_t$ is far from $\vx_0$.
By applying Lemma~\ref{lem:improve_or_localize} we will show that the objective sufficiently decreases.


\subsubsection{Difference Between Coupling Sequences}
In order to capture the difference between the coupling sequences, we introduce the following notation:
\begin{align*}
    \hat\vx_t = \vx'_t -\vx_t
    && \hat \err_t = \err'_t - \err_t
    && \hat \sgdNoise_t = \sgdNoise'_t - \sgdNoise_t
    && \hat \ourNoise_t = \ourNoise'_t - \ourNoise_t
    && \hat \tvx_t = \tvx'_t - \tvx_t
\end{align*}
Recall that $\sgdNoise_t$ is stochastic noise, $\ourNoise_t$ is artificial noise, $\err_t$ is the compression error.

\begin{definition}
\label{def:error-series}
Let
$\delta_i = \int_0^1 \nabla^2 f(\alpha \vx'_i + (1 - \alpha) \vx_i) d \alpha - H$.
Then
\begin{align*}
    \Delta_t &= \step \sum_{i=0}^{t-1} (I - \step H)^{t-i-1}\delta_i \hat \vx_i \\
    \mathcal E_t &= \step \sum_{i=0}^{t-1} (I - \step H)^{t-i-1} (\hat \err_{i} - \hat \err_{i+1}) \\
    Z_t &= \step \sum_{i=0}^{t-1} (I - \step H)^{t-i-1}\hat \sgdNoise_i \\
    \Xi_t &= \step \sum_{i=0}^{t-1} (I - \step H)^{t-i-1}\hat \ourNoise_i,
\end{align*}
\end{definition}
\begin{proposition}
    \label{prop:diff_decomposition}
    $\hat\vx_t = -(\Delta_t + \mathcal E_t + Z_t + \Xi_t)$.
\end{proposition}
In the simplest case, the objective is quadratic and we have access to an uncompressed deterministic gradient.
When it's not the case, the introduced terms show how the actual algorithm behavior is different:
\begin{compactitem}
    \item $\Delta_t$ corresponds to quadratic approximation error.
    \item $\mathcal E_t$ corresponds to compression error.
    \item $Z_t$ corresponds to difference arising from SGD noise.
    \item $\Xi_t$ corresponds to difference arising from artificial noise.
\end{compactitem}
Intuitively, $\Xi_t$ is a good term, and other terms are negligible ($\|\Delta_t + \mathcal E_t + Z_t\| < \frac 12 \|\Xi_t\|$).

\begin{proof}
\begin{align*}
    \hat \vx_{t+1}
    &= \vx'_{t+1} - \vx_{t+1} \\
    &= \tvx'_{t+1} + \step \err'_{t+1} - (\tvx_{t+1} + \step \err_{t+1})
        && \text{(By definition of $\tvx_t$ and $\tvx'_t$)} \\
    &= \step \hat \err_{t+1} + (\tvx'_t - \tvx_t) - \step \left((\nabla f(x'_t) - \nabla f(x_t)) + (\sgdNoise'_t - \sgdNoise_t) + (\ourNoise'_t - \ourNoise_t) \right) 
        && \text{(By update equation for $\tvx_t$)} \\
    &= \step (\hat \err_{t+1} - \hat \err_t) +  \hat\vx_t - \step \left((\delta_t + H) \hat\vx_t + \hat \sgdNoise_t + \hat \ourNoise_t \right) 
        && \text{(By definition of $\delta_t$ and $\tvx_t$)} \\
    &= \step (\hat \err_{t+1} - \hat \err_t) +  (I - \step H) \hat\vx_t - \step \left(\delta_t \hat\vx_t + \hat \sgdNoise_t + \hat \ourNoise_t \right) \\
    &= (I - \step H) \hat\vx_t - \step \left(\delta_t \hat\vx_t + (\hat \err_t - \hat \err_{t+1}) + \hat \sgdNoise_t + \hat \ourNoise_t \right) \\
\end{align*}
Using telescoping, we get the required expression.
\end{proof}
Since $\hat \tvx_t = \hat\vx_t - \step \hat \err_t$, we have:
\[\hat\vx_t = -(\Delta_t + \mathcal E_t + Z_t + \Xi_t)
\iff \hat\tvx_t = -(\Delta_t + (\mathcal E_t + \step \hat \err_t) + Z_t + \Xi_t),\]
and we'll use $\|\hat \tvx_t\|$ in Corollary~\ref{cor:escaping_improves}.

\subsubsection{Bounding Accumulated Compression Error}
\label{sec:bound_acc_compr_error}
Compared to SGD analysis, an additional term $\mathcal E_t + \step \hat \err_t$ appears.
This term corresponds to accumulated error arising from compression, and we have to bound it.

\begin{definition}
Following~\citet{jin2019nonconvex}, we introduce the following term:
\label{def:beta}
\begin{align*}
    \sumvar_t = \sqrt{\sum_{i=0}^{t-1} (1 + \step \eigen)^{2i}}
\end{align*}
\end{definition}

\begin{proposition}[\citet{jin2019nonconvex}, Lemma~29]
    \label{prop:beta_equals}
    If $\step \eigen \in [0, 1]$, then for all $t$: $\beta_t \le \frac {(1 + \step \eigen)^t} {\sqrt{2 \step \eigen}}$,
    and for all $t \ge \frac 2 {\step \eigen}$: $\beta_t \ge \frac {(1 + \step \eigen)^t} {\sqrt{6 \step \eigen}}$.
\end{proposition}

\begin{proposition}
    \label{prop:bound_series}
    For any $t \le \escIter$:
    \begin{align*}
        \left(\sum_{i=0}^{t-1} (1 + \step \eigen)^{t-1-i}\right)^2
        \le \frac {\sumvar_t^2} {\step \sqrtre}
    \end{align*}
\end{proposition}
\begin{proof}
    Using Cauchy-Schwarz:
    \begin{align*}
        \left(\sum_{i=0}^{t-1} (1 + \step \eigen)^{t-1-i}\right)^2
        \le t \sum_{i=0}^{t-1} (1 + \step \eigen)^{2(t-1-i)}
        \le \escIter \sumvar_t^2
        \le \frac {\sumvar_t^2} {\step \sqrtre}
    \end{align*}
\end{proof}

\begin{lemma}[Bounding accumulated compression error]
    \label{lem:bound_error_sum}
    Under Assumptions~\ref{ass:lip} and~\ref{ass:sg}, let $\fullStd$ be as in Definition~\ref{def:error-series}, $\mathcal{E}_t$ and $\hat \err_t$ be as in Definition~\ref{def:error-series}, $\sumvar_t$ be as in Definition~\ref{def:beta} and $\step$ and $\escRad$ as in Equation~\ref{eq:params}.
    Assume that $\err_0=0$, $\Exp{\|\tvx_t - \tvx_0\|^2} < \escRad^2$ for all $t \le \escIter$ and $\|\nabla f(\vx_0)\| \le 4 \gradLip \escRad$.
    Let $- \eigen$ be the smallest negative eigenvalue of $\nabla^2 f(\vx_0)$ such that $\eigen \ge \frac {\sqrtre} 2$.
    Then under Assumptions A, B, D, for $t \le \escIter$ we have:
    \begin{align*}
        \Exp{\|\mathcal E_t + \step \hat \err_t\|^2}
        \le \frac {20 \step^3 (1 - \compr) \gradLip^2 \fullStd^2 \sumvar_t^2} {\compr^2 \sqrtre}
    \end{align*}
\end{lemma}

\begin{proof}
    Expanding sum in $\mathcal E_t$ and using that $\hat \err_0=0$:
    \begin{align*}
        \mathcal E_t
        &= \step \sum_{i=0}^{t-1} (I - \step H)^{t-1-i} (\hat \err_{i} - \hat \err_{i+1}) && \text{(By Definition~\ref{def:error-series})} \\
        &= \step (- \hat \err_{t} + \sum_{i=1}^{t-1}(I - \step H)^{t-1-i} ((I - \step H) - I) \hat \err_i) && \text{(By telescoping)} \\
        &= - \step \hat \err_{t} + \step^2 H \sum_{i=1}^{t-1}(I - \step H)^{t-1-i} \hat \err_i \\
    \end{align*}

    We can now estimate $\Exp{\|\mathcal E_t + \step \hat\err_t^2\|}$. Since $-\eigen$ is the smallest negative eigenvalue of $H$, we have $\|I - \step H\| \le (1 + \step \eigen)$.
    \begin{align*}
        \Exp{\|\mathcal E_t + \step \hat \err_t\|^2}
        &= \Exp{\|\step^2 H \sum_{i=1}^{t-1}(I - \step H)^{t-1-i} \hat \err_i\|^2} \\
        &\le \step^4 \gradLip^2 \Exp{(\sum_i (1 + \step \eigen)^{t-1-i} \|\hat \err_i\|)^2} && \text{(By $\gradLip$-smoothness, $\lambda_{\max}(H) \le \gradLip$)}\\
        &\le 2 \step^4 \gradLip^2 (\sum_i (1 + \step \eigen)^{t-1-i})^2 \max_i \Exp{\|\hat \err_i\|^2} && \text{($\Exp{a b} \le \max(\Exp{a^2}, \Exp{b^2})$)}\\
        &\le 2 \step^4 \gradLip^2 t (\sum_i (1 + \step \eigen)^{t-1-i})^2 \max_i \Exp{\|\err'_i - \err_i\|^2} && \text{(By definition of $\hat \err_i$)} \\
        &\le 4 \step^4 \gradLip^2 t (\sum_i (1 + \step \eigen)^{t-1-i})^2 \max_i \Exp{\|\err'_i\|^2 + \|\err_i\|^2} && \text{(By Cauchy-Schwarz)} \\
        &\le 8 \step^4 \gradLip^2 t (\sum_i (1 + \step \eigen)^{t-1-i})^2 \max_i \Exp{\|\err_i\|^2} && \text{($\err_i$ and $\err_i'$ have same distribution)}
    \end{align*}

    
    Similarly to Lemma~\ref{lem:large_gradient}, we can show that $\Exp{\|\nabla f(\vx_i)\|^2} \le 40 \gradLip^2 \escRad^2$.
    Using corollary from Lemma~\ref{lem:error_estimation}, we have
    \begin{align*}
        \Exp{\|\err_{t}\|^2}
        & \le \frac {4 (1 - \compr)} {\compr^2} (\max_i \Exp{\|\nabla f(\vx_i)\|^2} + \fullStd^2)
            & \text{(By Lemma~\ref{lem:error_estimation})}\\
        & \le \frac {4 (1 - \compr)} {\compr^2} (40 \gradLip^2 \escRad^2 + \fullStd^2)
            & \text{(By assumption $\Exp{\|\nabla f(\vx_i)\|^2} \le 40 \gradLip^2 \escRad^2$)} \\
        & \le \frac {5 (1 - \compr) \fullStd^2} {\compr^2}
            & \text{(Selecting sufficiently small $\crad$ in the definition of $\escRad$)} \\
    \end{align*}

    Substituting this into the inequality for $\|\mathcal E_t + \step \hat \err_t\|$:
    \begin{align*}
        \Exp{\|\mathcal E_t + \step \hat \err_t\|^2}
        \le 4 \step^4 \gradLip^2 t \left(\sum_i (1 + \step \eigen)^{t-1-i}\right)^2 \cdot \frac {5 (1 - \compr)  \fullStd^2} {\compr^2}
        \le \frac {20 \step^4 (1 - \compr) \gradLip^2 \fullStd^2 \sumvar_t^2} {\compr^2 \step \sqrtre},
    \end{align*}
    where we bounded the series using Proposition~\ref{prop:bound_series}.
\end{proof}
\begin{lemma}[Bounding accumulated compression error \textbf{for linear compressor}]
    \label{lem:bound_error_sum_linear}
    Under\\ conditions of Lemma~\ref{lem:bound_error_sum} (except of $\err_0=0$), additionally assume that the compressor is linear (Definition~\ref{def:linear_compr}). When $\step \le \step_\std$, for $t \le \escIter$ we have:
    \begin{align*}
        \Exp{\|\mathcal E_t + \step \hat \err_t\|^2} &\le \frac {9 \step^3 (1 - \compr) \gradLip^2 \sumvar_t^2 r^2} {\compr^2 d \sqrtre}
    \end{align*}
\end{lemma}
Note that, compared with Lemma~\ref{lem:bound_error_sum}, the denominator has an additional $d$ term.

\begin{proof}
    \begin{align*}
        \hat \err_{t+1}
        &= \err_{t+1} - \err_{t+1}' \\
        &= \nabla F(\vx_t, \sgdRand_t) + \ourNoise_t + \err_t - \C(\nabla F(\vx_t, \sgdRand_t) + \ourNoise_t + \err_t, \comprRand_t) \\
        & \quad - (\nabla F(\vx'_t, \sgdRand_t) + \ourNoise'_t + \err'_t - \C(\nabla F(\vx'_t, \sgdRand_t) + \ourNoise'_t + \err'_t, \comprRand_t)) \\
        &= (\nabla F(\vx_t, \sgdRand_t) - \nabla F(\vx'_t, \sgdRand_t)) + (\ourNoise_t - \ourNoise'_t) + (\err_t - \err_t') \\
        & \quad - \C\left((\nabla F(\vx_t, \sgdRand_t) - \nabla F(\vx'_t, \sgdRand_t)) + (\ourNoise_t - \ourNoise'_t) + (\err_t - \err_t'), \comprRand_t\right) \\
        &= (\nabla F(\vx_t, \sgdRand_t) - \nabla F(\vx'_t, \sgdRand_t)) + \hat \ourNoise_t + \hat \err_t
            - \C\left((\nabla F(\vx_t, \sgdRand_t) - \nabla F(\vx'_t, \sgdRand_t)) + \hat \ourNoise_t + \hat \err_t, \comprRand_t\right)
    \end{align*}
    
    We estimating the norm of $\hat \err_t$ using linearity of $\C$:
    \begin{align*}
        &\ExpArg{\comprRand_t}{\|\hat \err_{t+1}\|^2 \mid \vx_t, \err_t, \sgdRand_t}\\
        &= \ExpArg{\comprRand_t}{\|(\nabla F(\vx_t, \sgdRand_t) - \nabla F(\vx'_t, \sgdRand_t)) + \hat \ourNoise_t + \hat \err_t
                   - \C\left((\nabla F(\vx_t, \sgdRand_t) - \nabla F(\vx'_t, \sgdRand_t)) + \hat \ourNoise_t + \hat \err_t, \comprRand_t\right)\|^2} \\
        &\le (1 - \compr) \Exp{\|(\nabla F(\vx_t, \sgdRand_t) - \nabla F(\vx'_t, \sgdRand_t)) + \hat \ourNoise_t + \hat \err_t\|^2} \\
    \end{align*}
    
    Similarly to the proof of Lemma~\ref{lem:error_estimation}, for any $\nu$ we have:
    \begin{align*}
        \Exp{\|\hat \err_{t+1}\|^2}
        &\le (1 - \compr) \Exp{(1 + \frac 1 \nu)\|(\nabla F(\vx_t, \sgdRand_t) - \nabla F(\vx'_t, \sgdRand_t)) + \hat \ourNoise_t\|^2 + (1 + \nu)\|\hat \err_t\|^2} \\
        &\le \frac 1\nu \sum_{i=0}^t  \left((1 - \compr) (1+\nu)\right)^{t-i+1} \Exp{\|(\nabla F(\vx_t, \sgdRand_t) - \nabla F(\vx'_t, \sgdRand_t)) + \hat \ourNoise_t\|^2}
    \end{align*}
    
    By selecting $\nu= \frac {\compr} {2 (1 - \compr)}$ and computing the sum of a geometric series, we have:
    \begin{align*}
        \Exp{\|\hat \err_{t+1}\|^2}
        &\le \frac {2 (1 - \compr)} {\compr} \sum_{i=0}^t  \left(1 - \frac {\compr} 2\right)^{t-i+1} \Exp{\|(\nabla F(\vx_t, \sgdRand_t) - \nabla F(\vx'_t, \sgdRand_t)) + \hat \ourNoise_t\|^2}\\
        &\le \frac {2 (1 - \compr)} {\compr} \sum_{i=0}^t  \left(1 - \frac {\compr} 2\right)^{t-i+1} \max_i \Exp{\|(\nabla F(\vx_i, \sgdRand_i) - \nabla F(\vx'_i, \sgdRand_i)) + \hat \ourNoise_i\|^2}\\
        &\le \frac {4 (1 - \compr)} {\compr^2} \max_i \Exp{\|(\nabla F(\vx_i, \sgdRand_i) - \nabla F(\vx'_i, \sgdRand_i)) + \hat \ourNoise_i\|^2} \\
        &\le \frac {4 (1 - \compr)} {\compr^2} \max_i \Exp{\|\nabla F(\vx_i, \sgdRand_i) - \nabla F(\vx'_i, \sgdRand_i)\|^2} + \frac {r^2} d + \|\hat \err_0\|^2
    \end{align*}
    
    Substituting this into bound for $\mathcal E + \step \hat \err_t$ and bounding the series by Proposition~\ref{prop:bound_series}:
    \begin{align*}
        \Exp{\|\mathcal E_t + \step \hat \err_t\|^2}
        &\le 2 \step^4 \gradLip^2 \left(\sum_i (1 + \step \eigen)^{t-1-i}\right)^2 \max_i \Exp{\|\hat \err_i\|^2} \\
        &\le \frac {8 \step^4 (1 - \compr) \gradLip^2 \sumvar_t^2} {\compr^2 \step \sqrtre} \left(\max_i \Exp{\|\nabla F(\vx_i, \sgdRand_i) - \nabla F(\vx'_i, \sgdRand_i)\|^2} + \frac {r^2} d\right)
    \end{align*}

    Note that in the last term, $\frac {r^2} d$ always dominates another term:
    \begin{enumerate}
        \item When Assumption~\ref{ass:lipsg} holds, we bound $\Exp{\|\nabla F(\vx_i, \sgdRand_i) - \nabla F(\vx'_i, \sgdRand_i)\|^2}$ with $\sgdLip^2 \Exp{\|\hat \vx\|^2} \le 4 \sgdLip^2 \escRad^2$. Since $\escRad = \crad \sqrt{\frac \eps \hesLip}$ and $r = \cnoise \frac {\eps} {\sqrt {\gradLip \step}}$, we can select constants $\crad$ and $\cnoise$ so that the second term dominates the first one.
        \item When Assumption~\ref{ass:lipsg} doesn't hold hold, we bound $\Exp{\|\nabla F(\vx_i, \sgdRand_i) - \nabla F(\vx'_i, \sgdRand_i)\|^2}$ as $c(\std^2 + \escRad^2)$. Again (and using that $\step \le \step_\std \le \frac {\eps_2} d$), we can select the constants so that the second term dominates.
    \end{enumerate}
    As a result, we achieve the required bound:
    \begin{align*}
        \Exp{\|\mathcal E_t + \step \hat \err_t\|^2} &\le \frac {9 \step^3 (1 - \compr) L^2 \sumvar_t^2 r^2} {\compr^2 \sqrtre}
    \end{align*}
\end{proof}




\subsubsection{Escaping From a Saddle Point}

We now show that, if a starting point is a saddle point, we move sufficiently far from it.
\begin{lemma}[Non-localization]
    \label{lem:nonlocalize}
    Under Assumptions~\ref{ass:lip} and~\ref{ass:sg}, let $\fullStd$ be as in Definition~\ref{def:error-series}, $\sumvar_t$ be as in Definition~\ref{def:beta} and $\step$ and $r$ as in Equation~\ref{eq:params}.
    If the compressor is not linear, assume that $\err_0$.
    Assume that $\eigen = -\lmin(\nabla^2 f(\vx_0)) > \frac{\sqrtre} 2$ and $\Exp{\|\tvx'_t - \tvx_0\|^2} < \escRad^2$ for all $t \le \escIter$.
    Then for all $t \le \escIter$, for some constant $c$:
    \[\Exp{\|\hat \tvx_t\|^2} \ge c \frac {\sumvar_t^2 \step^2 r^2}{d}\]
\end{lemma}

\begin{proof}
    To simplify the presentation, we use $c$ to denote constants, and it may change its meaning from line to line.
    \[\Exp{\|\hat \tvx_t\|^2} \ge (\max(0, \Exp{\|\Xi_t\| - \|\Delta_t\| - \|\mathcal E_t + \step \hat \err_t\| - \|Z_t\|}))^2\]
    
    We show that $\Exp{\Xi_t} = \Omega \left(\frac {\sumvar_t \step r}{\sqrt d}\right)$, and terms aside from $\Xi_t$ are negligible, namely that in expectation $\Exp{\|\Delta_t\|}, \Exp{\|\mathcal E_t + \step \hat \err_t\|}, \Exp{\|Z_t\|} \le \frac 1 {10} \Exp{\|\Xi_t\|}$\footnote{Most of the proof can go through if we consider $\Exp{\|\cdot\|^2}$ instead of $\Exp{\|\cdot\|}$. There is only one place in estimation of $\|\Delta_t\|$ which requires the first momentum.}.

    We prove the inequality by induction.
    The inequality holds for $t=0$ since all terms are $0$.

    \paragraph{Estimating $\Xi_t$.}
    Since $\Xi_t$ is a sum of independent Gaussians with variances $4 (1 + \step \eigen)^{2(t - i - 1)} \frac {\step^2 r^2} d$,
    its total variance is
    \begin{align*}
        \Exp{\|\Xi_t\|^2}
        = 4\frac {\step^2 r^2} d \sum_{i=0}^{t-1} (1 + \step \eigen)^{2i}
        = 4\frac {\step^2 r^2} d \sumvar^2_t,
    \end{align*}
    And for a zero-mean Guassian random variable, we know $\Exp{\|\Xi_t\|}^2 = \frac 2\pi \Exp{\|\Xi_t\|^2}$.
    Note that from the induction hypothesis it follows that $\Exp{\|\hat \tvx_t\|^2} \le 2 \Exp{\|\Xi_t\|^2} \le 8 \frac {\step^2 r^2 \sumvar_t^2}{d}$.
    
    \paragraph{Bounding $\Delta_i$.} By the Hessian Lipschitz property, $\Exp{\|\delta_i\|^2} \le 4 \hesLip^2 \escRad^2$, and by the induction hypothesis:
    \[\Exp{\|\hat \vx_i\|^2} \le 2 \Exp{\|\hat \tvx_i\|^2} + 2 \step^2 \Exp{\|\hat \err_i\|^2} \le c\frac {\step^2 r^2 \sumvar_i^2} {d}\]

    for $i \le t$ and for $\step$ selected as in Equation~\eqref{eq:params} (see proofs of Lemmas~\ref{lem:bound_error_sum} and~\ref{lem:bound_error_sum_linear}). Therefore:
    \begin{align*}
        \Exp{\|\Delta_t\|}
        &= \Exp{\|\step \sum_{i=0}^{t-1} (I - \step H)^{t-i-1} \delta_i \hat \vx_i\|} 
            && \text{(By Definition~\ref{def:error-series})} \\
        &\le \step \Exp{\sum_{i=0}^{t-1} \|I - \step H\|^{t-i-1} \cdot \|\delta_i\| \cdot \|\hat \vx_i\|}
            \\
        &\le \step \sum_{i=0}^{t-1} (1 + \step \eigen)^{t-i-1} \sqrt{\Exp{\|\delta_i\|^2} \cdot \Exp{\|\hat \vx_i\|^2}}
            && (\text{Cauchy-Schwarz})\\
        &\le c \step \hesLip \escRad \frac {\step r}{\sqrt d}  \Exp{\sum_{i=0}^{t-1} (1 + \step \eigen)^{t-i-1} \sumvar_i}
            && (\Exp{\|\delta_i\|^2} \le 4 \hesLip^2 \escRad^2 \text{ and } \Exp{\|\hat \vx_i\|^2} \le c\frac {\step^2 r^2 \sumvar_t^2} {d}) \\
        &\le c \step \hesLip \escRad \frac {\step r}{\sqrt d}  \Exp{\sum_{i=0}^{t-1} \frac{(1 + \step \eigen)^{t-1}}{\sqrt{\step \eigen}}}
            && \text{(Proposition~\ref{prop:beta_equals})} \\
        &\le c \step \hesLip \escRad \frac {\step r}{\sqrt d}  \escIter \sumvar_t
            && \text{(Proposition~\ref{prop:beta_equals}, another direction)} \\
        &\le c \frac {\step r \sumvar_t}{\sqrt d} (\step \hesLip \crad \sqrt {\frac {\eps} \hesLip} \cdot \frac \cit {\step \sqrtre})
            && (\escRad = \crad \sqrt {\frac {\eps} \hesLip} \text{ and } \escIter = \frac \cit {\step \sqrtre}) \\
        &\le c \crad \cit \frac {\step r \sumvar_t}{\sqrt d},
    \end{align*}
    and it suffices to choose $c \crad \cit \le \frac 1{40}$ so that $\Exp{\|\Delta_t\|} \le \frac 1{10}\Exp{\|\Xi_t\|}$.

    \paragraph{Bounding $\|\mathcal E_t + \step \hat \err_t\|$.}
    For a general compressor, by Lemma~\ref{lem:bound_error_sum} we know that
    \begin{align*}
        \Exp{\|\mathcal E_t + \step \hat \err_t\|^2}
        \le \frac {c \step^3 (1 - \compr) \gradLip^2 \fullStd^2 \sumvar_t^2} {\compr^2 \sqrtre}
    \end{align*}

    Using $\fullStd \le 2 r$, to show that $\Exp{\|\mathcal E_t + \step \hat \err_t\|}^2 \le \Exp{\|\mathcal E_t + \step \hat \err_t\|^2} \le \frac 1 {100} \Exp{\|\Xi_t\|}^2$, it suffices to guarantee that
    \begin{align*}
        \frac {\step^3 (1 - \compr) \gradLip^2 \fullStd^2 \sumvar_t^2} {\compr^2 \sqrtre}
            \le c \frac {\sumvar_t^2 \step^2 r^2}{d}
        \iff \step \le c \frac {\compr^2 \sqrtre r^2} {d (1 - \compr) \fullStd^2 \gradLip^2}
    \end{align*}
    
    Using that $\fullStd^2 \le 2 r^2$ for sufficiently large $\cnoise$, we have:
    \begin{align*}
        \step \le c \frac {\compr^2 \sqrtre} {d (1 - \compr) \gradLip^2}
    \end{align*}
    
    When the compressor is linear, by Lemma~\ref{lem:bound_error_sum_linear} we have:
    \begin{align*}
        \frac {\step^3 (1 - \compr) L^2 \sumvar_t^2 r^2} {\compr^2 d \sqrtre} \le c \frac{\sumvar_t^2 \step^2 r^2}{d}
        \iff \step \le c \frac {\compr^2 \sqrtre} {(1 - \compr) \gradLip^2}
    \end{align*}

    \paragraph{Bounding $\|Z_t\|$.} First, we consider the case when Assumption~\ref{ass:lipsg} doesn't hold (i.e. $\sgdLip = +\infty$).
    Since $Z_t$ is the sum of independent random variables:
    \begin{align*}
        \Exp{\|Z_t\|^2}
        \le \step^2 \sum_{i=0}^{t-1} (1 + \step \eigen)^{2 (t-i-1)} 2 \step^2 \std^2
        \le 2 \step^4 \sumvar_t^2 \std^2
    \end{align*}
    To prove that $\Exp{\|Z_t\|}^2 \le \Exp{\|Z_t\|^2} < \frac 1 {100} \Exp{\|\Xi_t\|}^2$, it suffices to show that
    \[\step^4 \sumvar_t^2 \std^2 \le c \frac {\sumvar_t^2 \step^2 r^2}{d}
    \iff \std\sqrt d \le c r
    \iff \std^2 d \le c \cnoise^2 \frac{\eps^2} {\gradLip \step}
    \iff \step \le c \frac {\cnoise^2 \eps^2} {\std^2 \gradLip d},\]
    which holds when $\cstep \le c \cnoise^2$, by Equation~\eqref{eq:params}.

    Finally, we consider the case when Assumption~\ref{ass:lipsg} holds (i.e.\ $\sgdLip < +\infty$).
    Since stochastic gradient is Lipschitz, we have $\|\hat \sgdNoise_i\| \le 2 \sgdLip \|\hat x_i\|$ and:
    \begin{align*}
        \Exp{\|Z_t\|^2}
        &= \Exp{\|\step \sum_{i=0}^{t-1} (I - \step H)^{t-i-1}\hat \sgdNoise_i\|^2}
            && \text{(Definition~\ref{def:error-series})}\\
        &\le \step^2 \sum_{i = 0}^{t-1} \Exp{\|(I - \step H)^{t-i-1}\hat \sgdNoise_i\|^2}
            && \text{(Noises are independent)} \\
        &\le \step^2 \sum_{i = 0}^{t-1} \|(1 + \step \eigen)^{t-i-1}\|^2 \cdot \Exp{\|\hat \sgdNoise_i\|^2}
            && \text{(Since $\eigen$ is the smallest negative eigenvalue of $H$)} \\
        &\le \step^2 \sum_{i = 0}^{t-1} \|(1 + \step \eigen)^{t-i-1}\|^2 \cdot \sgdLip^2 \Exp{\|\hat \vx_i\|^2}
            && \text{(Assumption~\ref{ass:lipsg})} \\
        &\le c \step^2 \escIter \frac {\step^2 r^2 \sumvar_t^2} {d}
            && \text{(See derivation for $\|\Delta_t\|$ above)}
    \end{align*}

    Therefore $\Exp{\|Z_t\|} \le c\step \sgdLip \sqrt \escIter \frac {\sumvar_t \step r} {\sqrt d}$.
    To guarantee that $\Exp{\|Z_t\|} \le \frac 1 {10} \Exp{\|\Xi_t\|}$, it suffices to show that
    \begin{align*}
        \step \sgdLip \sqrt{\escIter} \frac {\sumvar_t \step r} {\sqrt d} \le c \frac {\step r \sumvar_t} {\sqrt d}
        \iff \step \sgdLip \sqrt{\escIter} \le c
        \iff \frac {\cit^2 \step \sgdLip^2} \sqrtre \le c
        \iff \step \le c \frac \sqrtre {\cit^2 \sgdLip^2},
    \end{align*}
    which holds when $\cit^2 \cstep \le c$.
%
%
\end{proof}

\begin{theorem}
    Under Assumptions~\ref{ass:lip} and~\ref{ass:sg}, for $\step$ as in Equation~\eqref{eq:params}, after $\tilde O\left(\frac \fmax {\step \eps^2}\right)$ iterations of Algorithm~\ref{alg:arbitrary_compressor}:
    \begin{compactitem}
        \item For a linear compressor and $\reseterr=false$, at least half of visited points $\vx_{t}$ are $\eps$-SOSP\@.
        \item For a general compressor and $\reseterr=true$, at least half of points $\vx_{t}$ such that the condition at Line~\ref{line:check_err_to_zero} is triggered at iteration $t$ are $\eps$-SOSP\@.
        The condition is triggered at most $\tilde O(\frac \fmax \df) = \tilde O(\frac T \escIter)$ times.
    \end{compactitem}
\end{theorem}

Note that the fraction of $\eps$-SOSP can be made arbitrary close to $1$.

\begin{proof}
    As in the previous Lemma, $c$ is used to denote constants and may change its meaning from line to line,
    If at least quarter of the points have gradient $\|\nabla f(\vx_t)\| \ge \eps$, then by Lemma~\ref{lem:descent_lemma}, $\Exp{\tvx_0 - \tvx_T} > \fmax$, which is impossible.
    It remains to show that there is at most quarter of points such that $\lmin(\nabla^2(f(\vx_t))) \ge -\sqrtre$.
    First we show that, if $\lmin(\nabla^2 f(\vx_0)) > -\frac \sqrtre 2$, then for some $t \le \escIter$
    \[f(\vx_0) - \Exp{f(\vx_t)} \ge \df\]
    By Lemma~\ref{lem:nonlocalize}:
    \[\Exp{\|\hat \tvx_t\|^2}
    \ge c \frac {\sumvar_t^2 \step^2 r^2}{d}
    \ge c \frac {(1 + \step \eigen)^{2t} \step^2}{d \step \eigen} \cdot \frac{\eps^2}{\gradLip \step}
    \ge c \frac {(1 + \step \eigen)^{2t} \eps^2}{d \eigen \gradLip}
    \]
    Substituting $t = \escIter$, we have
    $(1 + \step \eigen)^\escIter
    \ge (1 + \step \sqrt{\rho \eps})^{\nicefrac{\cit}{\step \sqrt{\rho \eps}}}
    \ge e^{\cit}$.
        By selecting $\cit \ge c \log \frac {d \gradLip \hesLip 2 \escRad^2} {\std \eps}$ for some $c$, we have $\Exp{\|\hat \tvx_t\|^2} \ge 2 \escRad^2$, and therefore:
    \[\Exp{\|\tvx_0 - \tvx_\escIter\|^2} = \max(\Exp{\|\tvx_0 - \tvx_\escIter\|^2}, \Exp{\|\tvx_0 - \tvx'_\escIter\|}^2) \ge \frac 12 \Exp{\|\tvx'_\escIter - \tvx_\escIter\|^2} \ge \escRad^2\]
    Since by Proposition~\ref{prop:same_distribution} $\tvx_t$ and $\tvx'_t$ have the same distribution, $\Exp{\|\tvx_0 - \tvx_\escIter\|} = \Exp{\|\tvx_0 - \tvx'_\escIter\|}$, and therefore
    \[\Exp{\|\tvx_0 - \tvx_\escIter\|^2} \ge \escRad^2,\]
    
    and by Corollary~\ref{cor:escaping_improves}:
    \[f(\vx_0) - \Exp{f(\vx_\escIter)} \ge \df,\]
    
    and therefore the objective decreases by $\df$ after $\escIter$ iterations.
    
    For a linear compressor, we split consider iteration $s_1, s_2, \ldots$ such that $s_{i+1} = s_i + \escIter$.
    If among such iterations at least quarter of the points have $\lmin(\nabla^2 f(\vx_{s_i})) \ge -\frac \sqrtre 2$, then
    \[\Exp{\tvx_0 - \tvx_T}
        > \frac 14 \frac T\escIter \df
        \ge \frac c4 \frac {\fmax \step \sqrtre}{\step \eps^2} \sqrt {\frac {\eps^3}{\hesLip}}
        \ge \frac c4 \fmax,\]
    which is impossible by selecting a sufficiently large constant in the choice of $T$.
    By considering $s_i = 1, \ldots, \escIter$, we get that at most quarter of all points have $\lmin(\nabla^2 f(\vx_{s_i})) \ge -\sqrtre$.

    For an arbitrary compressor, the reasoning is similar. We consider points $\vx_{s_1}, \vx_{s_2}, \ldots$ such that the condition at Line~\ref{line:check_err_to_zero} of Algorithm~\ref{alg:arbitrary_compressor} is triggered at iteration $s_i$.
    Then, either we have escaped the saddle point (and therefore the objective decreased by $\df$) or at most $\escIter$ iterations passed. Similarly to the above, at most a quarter of such points have $\lmin(\nabla^2 f(\vx_{s_i})) \ge -\sqrtre$.
\end{proof}



\end{document}